\documentclass{article}

\usepackage[final,nonatbib]{neurips_2022}

\usepackage[ruled,vlined,linesnumbered]{algorithm2e}
\usepackage{amsfonts}       
\usepackage{amsmath}
\usepackage{amsthm}
\usepackage{booktabs}       
\usepackage{bm}
\usepackage{caption}
\usepackage{enumitem}
\usepackage[T1]{fontenc}    
\usepackage{graphicx}
\usepackage{hyperref}       
\usepackage[utf8]{inputenc} 
\usepackage{nicefrac}       
\usepackage{makecell}
\usepackage{mathtools}
\usepackage{microtype}      
\usepackage{multirow}
\usepackage{subcaption}
\usepackage{url}            
\usepackage{wrapfig}
\usepackage{xcolor}         

\newcommand{\Bzero}{\bm{0}}
\newcommand{\Bone}{\bm{1}}

\newcommand{\Bv}{\bm{v}}

\newcommand{\BA}{\bm{A}}

\newcommand{\BB}{\bm{B}}

\newcommand{\BBbar}{\smash{\overline{\bm{B}}}}

\newcommand{\BC}{\bm{C}}

\newcommand{\BD}{\bm{D}}

\newcommand{\BI}{\bm{I}}
\newcommand{\BK}{\bm{K}}
\newcommand{\BKhat}{\smash{\hat{\bm{K}}}}
\newcommand{\BM}{\bm{M}}

\newcommand{\BQ}{\bm{Q}}
\newcommand{\BQhat}{\smash{\hat{\bm{Q}}}}

\newcommand{\BS}{\bm{S}}

\newcommand{\BShat}{\smash{\hat{\bm{S}}}}

\newcommand{\BV}{\bm{V}}
\newcommand{\BW}{\bm{W}}
\newcommand{\BX}{\bm{X}}

\newcommand{\BY}{\bm{Y}}
\newcommand{\BYbar}{\smash{\overline{\bm{Y}}}}
\newcommand{\BZ}{\bm{Z}}
\newcommand{\BZbar}{\smash{\overline{\bm{Z}}}}

\newcommand{\loss}{\mathcal{L}}

\newcommand{\R}{\mathbb{R}}

\newcommand{\E}{\mathbb{E}}

\newcommand{\calE}{\mathcal{E}}
\newcommand{\calO}{\mathcal{O}}

\newcommand{\diag}{\text{diag}}

\DeclarePairedDelimiter\floor{\lfloor}{\rfloor}

\newcommand{\wbigcup}{\mathop{\bigcup}\displaylimits}

\newcommand{\cutsectionup}{\vspace*{-0.15in}}
\newcommand{\cutsectiondown}{\vspace*{-0.12in}}
\newcommand{\cutsubsectionup}{\vspace*{-0.1in}}
\newcommand{\cutsubsectiondown}{\vspace*{-0.07in}}
\newcommand{\cutparagraphup}{\vspace*{-0.1in}}

\newcommand{\modelname}[0]{SBM-Transformer}
\newcommand{\sungjun}[1]{\textcolor{blue}{Sungjun says: #1}}

\newtheorem{theorem}{Theorem}
\newtheorem{lemma}{Lemma}
\newtheorem{assumption}{Assumption}

\title{Transformers meet Stochastic Block Models:\\Attention with Data-Adaptive Sparsity and Cost}

\author{%
  Sungjun Cho$^1$ \quad Seonwoo Min$^1$ \quad Jinwoo Kim$^2$\\
  \textbf{Moontae Lee}$^{1,3}$ \quad \textbf{Honglak Lee}$^1$ \quad \textbf{Seunghoon Hong}$^{2,1}$\\
  $^1$LG AI Research \quad $^2$KAIST \quad $^3$University of Illinois Chicago\\
}

\begin{document}

\maketitle

\begin{abstract}
  To overcome the quadratic cost of self-attention, recent works have proposed various sparse attention modules, most of which fall under one of two groups: 1) sparse attention under a hand-crafted patterns and 2) full attention followed by a sparse variant of softmax such as $\alpha$-entmax. Unfortunately, the first group lacks adaptability to data while the second still requires quadratic cost in training. In this work, we propose SBM-Transformer, a model that resolves both problems by endowing each attention head with a mixed-membership Stochastic Block Model (SBM). Then, each attention head data-adaptively samples a bipartite graph, the adjacency of which is used as an attention mask for each input. During backpropagation, a straight-through estimator is used to flow gradients beyond the discrete sampling step and adjust the probabilities of sampled edges based on the predictive loss. The forward and backward cost are thus linear to the number of edges, which each attention head can also choose flexibly based on the input. By assessing the distribution of graphs, we theoretically show that SBM-Transformer is a universal approximator for arbitrary sequence-to-sequence functions in expectation. Empirical evaluations under the LRA and GLUE benchmarks demonstrate that our model outperforms previous efficient variants as well as the original Transformer with full attention. Our implementation can be found in \url{https://github.com/sc782/SBM-Transformer}.
\end{abstract}

\section{Introduction}\label{sec:introduction}
The Transformer \cite{transformer} architecture has been the go-to method for encoding sequential data, due to its superior performance in various tasks such as machine translation \cite{Transformer-NMT}, image classification \cite{ViT}, and protein language modeling \cite{MSATranformer}. Its key strength stems from the multi-head attention module, where a so-called {\it attention score} matrix computes how contextually important one token is to another for all possible token pairs. Each Transformer layer simultaneously pools the token representations based on the attention scores, eventually returning contextualized features without sequentially traversing through the input sequence as its recurrent neural network-based predecessors \cite{LSTM}.

A well-known drawback of the original Transformer is its high computational cost in time and memory that increases quadratically with sequence length. This is due to the full pairwise computation of attention scores, which prohibits applying it in tasks involving long-range dependencies such as document summarization \cite{huang2021efficient} or high-resolution image processing \cite{zhang2021multi}. Many works have thus focused on developing more efficient alternatives by exploiting fixed or learnable attention sparsity patterns \cite{sparsetransformer,bigbird,reformer,smyrf}, low-rank approximations \cite{linformer,nystromformer}, or kernelized attention modules~\cite{lineartransformer, performer}. 


Even though the efficient alternatives hold theoretical expressibility guarantees~\cite{yun2020}, they are far from sufficient, still failing to convince practitioners to replace the original Transformer. We believe this is mostly due to their lack of adaptability. They apply the same modifications to unanimously sparsify all the attention modules across layers, without considering the tasks at hand. Such strategy imposes inductive bias too strongly and often leads to sub-optimal cost vs. performance trade-offs in downstream tasks~\cite{narang2021}. In this work, we argue that to retain the utmost potential of Transformers, each attention module should have the ability to flexibly choose between sparse and full attention. This is especially evident when considering many state-of-the-art systems suggest the need for a mixture of dense and sparse attention layers. For example, a qualitative analysis on pretrained BERT showed that lower layers exhibit broad dense attention while upper layers perform focused sparse attention~\cite{clark2019}.  In the case of GPT-3~\cite{gpt-3}, the Transformer blocks are manually arranged to alternate between dense and sparse attention.

To contribute to the efficient Transformers lineage, we propose \modelname{}, capable of adjusting its attention sparsity data-adaptively based without fully computing the attention score matrix (Figure~\ref{fig:overall}). Leveraging a mixed-membership Stochastic Block Model (SBM)~\cite{mmsbm}, each attention head samples a bipartite graph connecting queries to keys. Then, the adjacency of the sampled graph is used as an attention mask so that only attention scores corresponding to sampled edges are computed. 
The overall computational cost is linear in the number of edges, which can range from linear to quadratic in sequence length depending on the data and task under concern. Each attention head is equipped with its own underlying SBM, enabling the model to diversify the attention sparsity across heads and layers. By incorporating a straight-through estimator~\cite{ste} in the discrete graph-sampling step, \modelname{} enjoys end-to-end differentiability and can find the proper attention sparsity based solely upon minimizing the predictive loss. The model can also easily be further regularized by penalizing the number of sampled edges, which results in a lighter model using less computational resources during inference. To the best of our knowledge, our method is the first Transformer architecture that can data-adaptively choose between linear to full attention with respective computational costs. To summarize, our main contributions are as follows:

\begin{figure}[!t]
    \centering
    \includegraphics[trim={0mm 116mm 67mm 0 },clip,width=\textwidth]{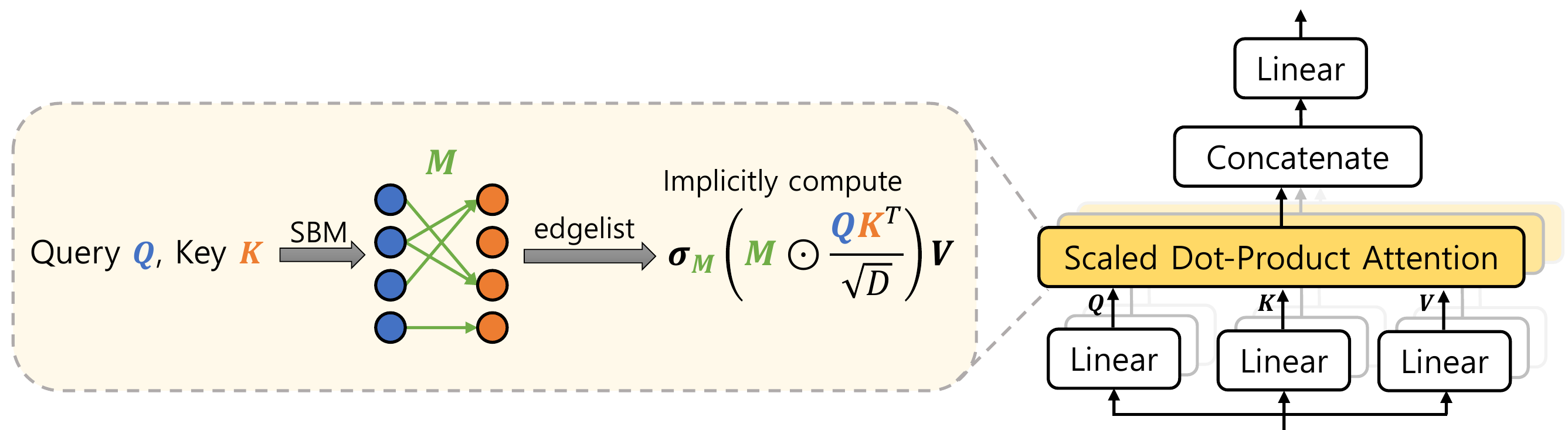}
    \vspace{-5mm}
    \caption{The attention module in SBM-Transformer. In multi-head attention, each attention head samples a bipartite graph connecting queries to keys from an underlying SBM. The adjacency of the sampled graph is used as an attention mask to compute the dot products only for the sampled edges.}
    \label{fig:overall}
    \vspace{-4mm}
\end{figure}

\begin{itemize}
    \item We present \modelname{}, a novel Transformer of which each attention head can adaptively adjust its attention sparsity as well as computational cost based on the input data.
    \item To demonstrate the benefit of this flexibility, we theoretically prove that \modelname{} retains universal approximability, and also stress-test the model under a synthetic task where full attention is required to achieve 100\% accuracy.
    \item Evaluations on LRA and GLUE benchmarks show that \modelname{} outperforms previous efficient Transformer models as well as the vanilla Transformer with dense attention.
\end{itemize}

\cutsubsectionup
\section{Related Work}\label{sec:related}
\cutsubsectiondown
In this section we discuss previous efficient Transformer variants and several works similar to ours with respect to adaptively learning sparse attention patterns. We also review several works on SBMs.

\cutparagraphup
\paragraph{Efficient Transformers.} Many efficient Transformers tackle to reduce the quadratic cost of multi-head attention with different approaches. While we discuss only a handful of representative approaches, a much more comprehensive survey can be found in~\cite{etsurvey}. The Linear Transformer~\cite{lineartransformer} achieves linear complexity by replacing the softmax with a low-rank kernelized function. Linformer~\cite{linformer} and Nystr\"omformer~\cite{nystromformer} use a similar approach by low-rank approximating the attention score matrix. Performer~\cite{performer} uses positive orthogonal random features to approximate the softmax kernel. Reformer~\cite{reformer} gathers similar tokens together through locality-sensitive hashing (LSH) and performs attention amongst tokens within the same bucket. Of all methods above, our method is most similar to Reformer, in the sense that we adaptively assign queries and keys into clusters and form a low-rank sparse attention pattern. However, our method performs soft-clustering with much less structural constraints, allowing each attention head to represent a wider variety of dependency structure and to adjust its sparsity towards full attention if needed.

\cutparagraphup
\paragraph{Adaptive Sparsity.} With respect to flexible training between sparse and dense attention, there exist some works that parameterize how sparse the attention pattern should be based on the input. The Adaptive Sparse Transformer~\cite{correia2019} proposed replacing the usual softmax activation with $\alpha$-entmax, in which the $\alpha$ parameter can be differentiably trained to adjust the activation between softmax and sparsemax activation~\cite{martins2020}. SparseBERT~\cite{sparsebert} uses a differentiable masking technique where each attention mask is sampled from a Gumbel-sigmoid distribution using data-independent mask probability parameters.
While these methods 
possess the flexibility to adjust between sparse and full attention based on data, they still require full computation of the attention score matrix before sparsification, and hence are unable to leverage the learned sparsity towards better model efficiency. To the best of our knowledge, ours is the first work to be able to adaptively tune its attention sparsity between sparse to full attention without requiring the explicit computation of the attention score matrix, thereby avoiding quadratic cost when possible. 

\cutparagraphup
\paragraph{Stochastic Block Models.} The Stochastic Block Model (SBM) is a generative model that encodes the latent structure of graphs by grouping nodes into clusters. By modeling the cluster-membership of each node as well as inter-cluster relationships, SBMs can represent a wide variety of graph structures, which is a feature especially useful for generating new graphs or predicting missing edges in noisy data~\cite{abbe2017community}. The standard SBM assigns each node to a single cluster, and the probability of an edge between two nodes strictly depends on the corresponding clusters. Several structural extensions include overlapping SBM~\cite{osbm} and mixed-membership SBM~\cite{mmsbm}, which allow each node to be assigned to multiple clusters. The underlying SBM used by our framework mostly resembles these two variants, while the edge probability is modeled by a nonlinear function of two node embeddings rather than a bilinear one. There exist many other extensions including degree-corrected SBM~\cite{dcsbm} for multi-graphs and hierarchical SBM~\cite{hsbm} for multiplex-graphs. Further details can be found in a recent survey~\cite{funke2019sbm}.


\cutsubsectionup
\section{Preliminaries: Sparse Transformers}\label{sec:background}
\cutsubsectiondown

We first introduce the full attention mechanism used in the original Transformer~\cite{transformer} as well as masked attention which will serve as a backbone of our approach.

\cutparagraphup
\subsection{Full Attention}
\cutsubsectiondown
In vanilla Transformer \cite{transformer}, each attention head takes a sequence of token features as input $\BX \in \R^{n \times d}$ where $n$ is the sequence length and $d$ the embedding dimension. Weight parameters $\BW^Q, \BW^K \in \R^{d \times d_h}$ and $\BW^V \in \R^{d \times d_h}$ with head-dimension $d_h$ first maps the input features $\BX$ into query $\BQ$, key $\BK$, and value $\BV$, respectively. Then, the {\it attention score matrix} is computed with scaled dot-product of queries and keys followed by row-wise softmax activation $\sigma(\cdot)$. Note that explicit computation of this matrix is the main bottleneck of full attention, incurring $\calO(n^2)$ asymptotic cost in both time and memory. The value features $\BV$ are then pooled based on the attention scores, returning the output token representations. Altogether, the operation performed by each attention head can be written as
\begin{gather}
    \BQ = \BX \BW^Q,\;\; \BK = \BX \BW^K,\;\; \BV = \BX \BW^V\\
    \texttt{Attn}(\BX) = \sigma\left(\dfrac{\BQ\BK^T}{\sqrt{d_h}}\right) \BV.
\end{gather}

\cutsectionup
\subsection{Masked Attention}
\cutsubsectiondown
One way to remove the quadratic bottleneck from the attention score matrix is to apply a binary mask $\BM \in \{0,1\}^{n \times n}$ and compute the scaled dot-products $\BQ_i\BK_j^T/\sqrt{d_h}$ only if $\BM_{ij} = 1$. In presence of an attention mask, the operation is modified to
\begin{gather}
    \texttt{Attn}_{\text{mask}}(\BX, \BM) = \sigma_{\BM}\left(\BM \odot \dfrac{\BQ\BK^T}{\sqrt{d_h}}\right) \BV\\
    \sigma_{\BM}(\BA)_{ij} \coloneqq 
    \begin{cases}
        \dfrac{\exp(\BA_{ij})}{\sum_{k \in \{k' | \BM_{ik'} = 1\}} \exp(\BA_{ik})} & \text{if}\;\; \BM_{ij} = 1\\
        \hfil 0 & \text{otherwise}
    \end{cases}
\end{gather}
where $\odot$ indicates entry-wise multiplication. Note that the masked-softmax $\sigma_{\BM}(\cdot)$ operator only computes unmasked terms, ensuring that each $(i,j)$-th attention score survives as nonzero if and only if $\BM_{ij} = 1$. This is thus equivalent to filling in the $(i,j)$-th attention score with $-\infty$ if $\BM_{ij}=0$, then applying the standard softmax operator. Most sparsity-based efficient Transformers fall under this formulation, while using different methods to either manually fix or learn the mask $\BM$. For instance, local attention~\cite{sparsetransformer,longformer,bigbird} with a sliding window sets $\BM_{ij} = 1$ if $|i-j|<c$ for some context window size $c$ while Reformer~\cite{reformer} sets $\BM_{ij} = 1$ if $\BQ_i$ and $\BK_j$ are hashed into the same bucket. 

\cutsectionup
\section{Our Method: SBM-Transformer}\label{sec:method}
\cutsectiondown
\begin{figure}[!t]
    \centering
    \includegraphics[trim={5mm 129mm 140mm 0 },clip,width=\textwidth]{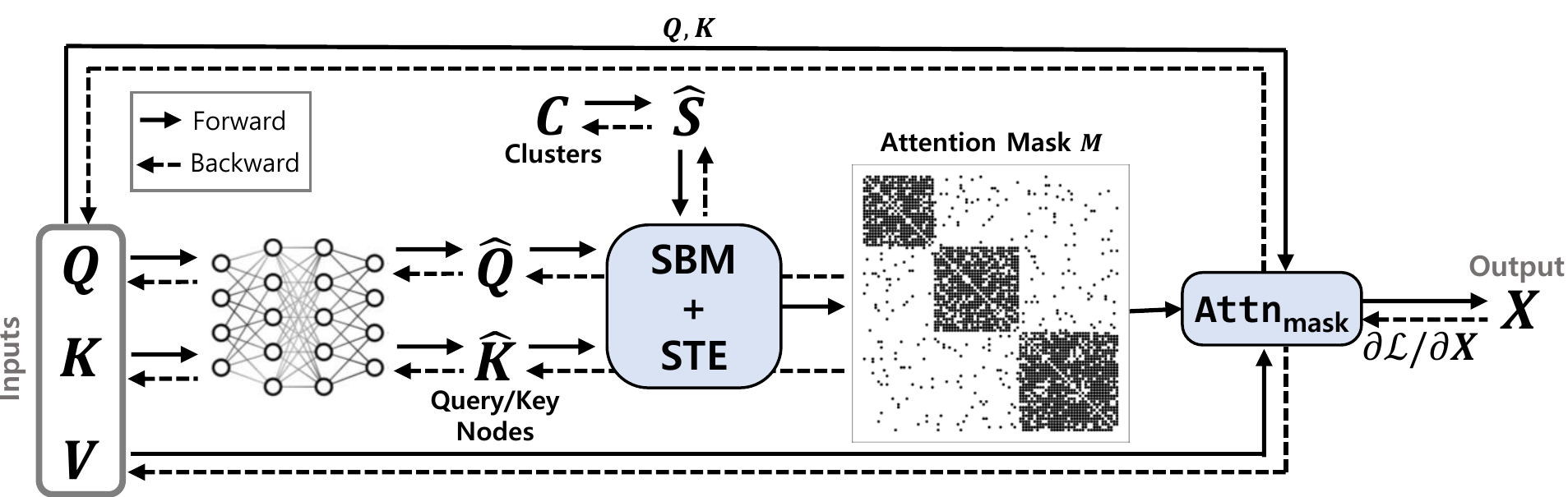}
    \vspace{-5mm}
    \caption{An illustration of the attention mechanism in SBM-Transformer. Each head first maps queries and keys to the node representation space through a shared MLP. The graph sampling module samples an attention mask from a Stochastic Block Model (SBM) parameterized by the node and cluster embeddings. The discrete sampling step is differentiable via a Straight-Through Estimator (STE). Given the mask, the output is computed via masked attention.}
    \vspace{-5mm}
    \label{fig:framework}
\end{figure}

Here we discuss the details of \modelname{} (Figure~\ref{fig:framework}). We first illustrate the forward step of our attention module and how the underlying SBM~\cite{mmsbm} of each head, from which we sample our attention masks, is parameterized by the input tensors. We then discuss how the model enables end-to-end differentiability despite the discrete graph sampling step.

\cutsectionup
\subsection{Forward step with the Stochastic Block Model}
\cutsectiondown

In our framework, we view the attention mask $\BM$ as an adjacency matrix of a bipartite graph that connects queries to keys, and let each attention head sample an adjacency matrix that best represents the contextual dependencies amongst input tokens.
In order to efficiently sample adjacency matrices while avoiding the quadratic cost, the distribution of graphs must first be parameterized with a sub-quadratic number of latent variables. Stochastic Block Models fit perfectly for our purpose as it models graphs that are low-rank structured with $k$ latent clusters, allowing full parameterization using $\calO(nk)$ memory. More concretely, the SBM distribution is defined by two nonnegative node-to-cluster memberships $\BY,\BZ \in \R_+^{n\times k}$ and a so-called block matrix $\BB \in \R_+^{k \times k}$ that stores the inter-cluster connection probabilities. The probability of node $i$ being connected to node $j$ is computed as $p(i,j) = \BY_i \BB \BZ_j^T$. Equivalently, the expectation of the adjacency matrix sampled from $\BA \sim SBM(\BY,\BB,\BZ)$ can be written as $\E[\BA] = \BY \BB \BZ^T$. 

%
For proper parameterization of the SBM, we must infer the nonnegative node-memberships and block matrix from the queries and keys. To do so, we equip each attention head a 2-layer $\text{MLP}_{d_h \to d_h}$ with ReLU activation, and a set of $k$ trainable cluster-embeddings $\BC \in \R^{k \times d_h}$. First, our model computes the block matrix $\BShat \in \R_+^{k \times k}$ by taking dot products amongst cluster-embeddings $\BC$ followed by a 2-dimensional softmax activation. The node embeddings are obtained by processing each query and key through the $\text{MLP}_{d_h \to d_h}$, mapping token representations into the node representation space.
The memberships of query and key nodes, which we denote by $\BQhat$ and $\BKhat$, are then inferred by taking dot products of node and cluster embeddings, followed by a sigmoid function. The block matrix $\BShat$, query node-memberships $\BQhat$, and key node-memberships $\BKhat$ altogether provide a well-defined parameterization for the SBM. Thus, a bipartite graph adjacency $\BM \in \{0,1\}^{n \times m}$ can be sampled from $\BM \sim SBM(\BQhat, \BShat, \BKhat)$ with expectation $\E[\BM] = \BQhat \BShat \BKhat^T$: the probability of connecting query $\BQ_i$ to key $\BK_j$ equals $p(i,j) = \BQhat_i \BShat \BKhat_j^T$. Formally, the sampling procedure can be written as

\begin{align}
    \BShat &= \texttt{softmax}(\BC \BC^T)\\
    \BQhat &= \texttt{sigmoid}(\text{MLP}_{d_h \to d_h}(\BQ) \BC^T)\\
    \BKhat &= \texttt{sigmoid}(\text{MLP}_{d_h \to d_h}(\BK) \BC^T)\\
    \BM &\sim SBM(\BQhat, \BShat, \BKhat)
\end{align}


\IncMargin{1.5em}
\begin{algorithm}[!t]
  \SetAlgoLined
  \SetKwInOut{Input}{Input}
  \SetKwInOut{Output}{Output}
  \Indm 
  \Input{$\BY \in \R^{n \times k}_+$, $\BB \in \R^{k \times k}_+$, $\BZ \in \R^{n \times k}_+$}
  \Output{$\BM \in \{0,1\}^{n \times n}$ with $\E[\BM] = \BY\BB\BZ^T$}
  \Indp 
  Compute diagonal matrices $\BD_{\BY} = (\diag(\Bone\BY))^{-1}$ and $\BD_{\BZ} = (\diag(\Bone\BZ))^{-1}$\\
  Column-normalize $\BYbar = \BY \BD_{\BY}^{-1}$ and $\BZbar = \BZ \BD_{\BZ}^{-1}$\\
  Compute $\BBbar = \BD_{\BY} \BB \BD_{\BZ}$\\
  Sample number of edges $m \sim \text{Poisson}(\Bone\BBbar\Bone^T)$\\
  Initialize $\BM = \Bzero$\\
  \For{$i = 1:m$}{
    Sample $(U,V)$ from $\{1,\dots,k\} \times \{1,\dots,k\}$ with $Pr(U=u,V=v) \propto \BBbar_{uv}$\\
    Sample source $I$ from $\{1,\dots,n\}$ with $Pr(I=i) = \BYbar_{iU}$.\\
    Sample destination $J$ from $\{1,\dots,n\}$ with $Pr(J=j) = \BZbar_{jV}$\\
    Set $\BM_{IJ} = 1$.
  }
  \caption{\texttt{fastRG}$(\BY, \BB, \BZ)$\cite{fastrg}}
  \label{alg:fastRG}
\end{algorithm}
\DecMargin{1.5em}

For the last sampling step, we incorporate a fast random graph sampling algorithm $\texttt{fastRG}$ (Alg.~\ref{alg:fastRG}, \cite{fastrg}) that can sample graphs from a SBM in time and memory asymptotically linear in the number of edges. 
One advantage of $\texttt{fastRG}$ is that each edge can be sampled in parallel, allowing high efficiency with the help of multiprocessing. A more significant feature of the method is that the number of edges, which determines the overall cost, is sampled from a Poisson distribution with input-dependent mean (Line 4). Thus, the model can dynamically adjust its computational cost between linear and quadratic in sequence length based on the data.

Figure~\ref{fig:sbm_sparsity} shows example placements of nodes and clusters on the $d_h$-dimensional space to show how the sparse structure is determined. If all nodes and clusters are gathered closely, then all entries in $\BQhat$ and $\BKhat$ become close to 1, resulting in $p(i,j) \approx 1$ for all $i,j$ and hence a dense $\BM$. If clusters are well-separated but each surrounded by some set of nodes, $\BShat$ becomes close to diagonal while each row in $\BQhat$ and $\BKhat$ is close to a one-hot vector indicating the cluster nearby. Such setting leads to a block diagonal mask similar to LSH bucketing of Reformer~\cite{reformer}. Lastly, if all clusters are far apart from the nodes, both $\BQhat$ and $\BKhat$ approximately equal zero, zeroing out all the edge probabilities.

\subsection{Backward Step with Straight-Through Estimator}


The graph sampling procedure is naturally a discrete operation. Thus, naive backpropagation cannot learn the proper parameterization for the SBM that minimizes the predictive loss. To cope with this non-differentiability, we incorporate a Straight-Through Estimator (STE)~\cite{ste} to pass the gradient beyond the discrete sampling step. The STE enables providing the gradient $\partial\loss/\partial\BM_{ij}$ to the probability for each sampled edge $(i,j)$ (Eqn.~\ref{eqn:ste}). It works as if we had used a continuous mask $\BM \odot \E[\BM]$ that stores the probability of each sampled edge instead of the binary mask $\BM$ during forward propagation. This way, the probabilities of sampled edges can be learned end-to-end: the gradients provide information on whether each sampled edge was useful or not for prediction.


\begin{align}
    \dfrac{\partial \loss}{\partial p_{ij}} \coloneqq \dfrac{\partial \loss}{\partial \BM_{ij}} =
    \begin{cases}
        \dfrac{\partial\loss}{\partial \BA_{ij}} \cdot \dfrac{\BQ_i \BK_j^T}{\sqrt{d_h}} & \text{if } \BM_{ij} = 1\\
        \hfil 0 & \text{otherwise}
    \end{cases}
    \text{ where }
    \BA \coloneqq \BM \odot \dfrac{\BQ\BK^T}{\sqrt{d_h}} \label{eqn:ste}
\end{align}


\begin{figure}[!t]
    \begin{center}
        \includegraphics[width=\linewidth]{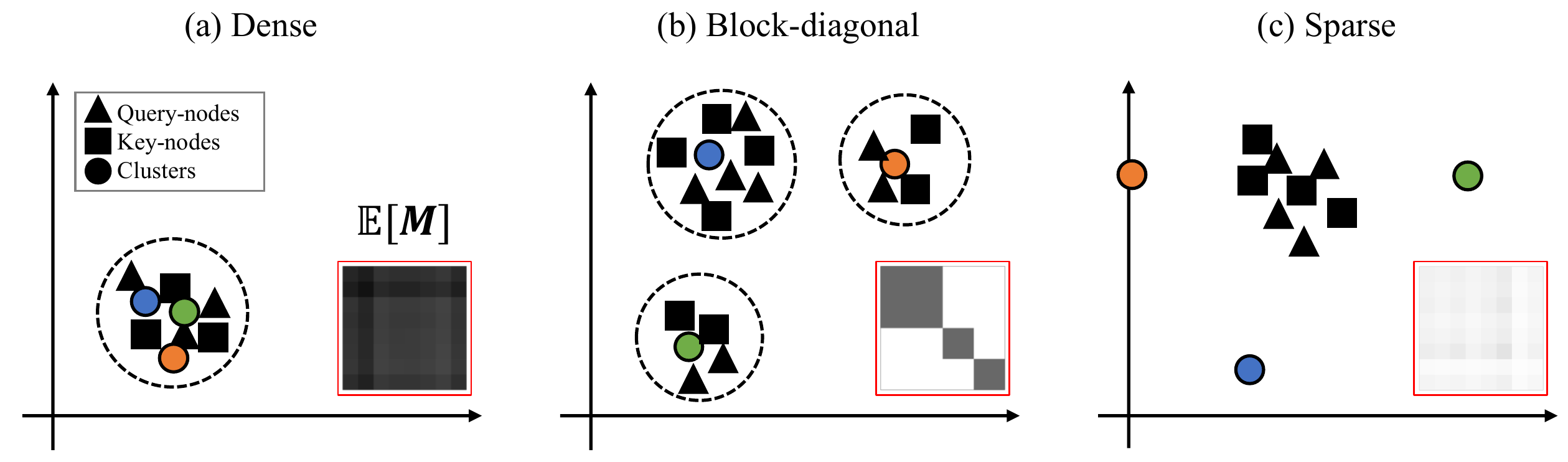}
    \end{center}
    \vspace{-3mm}
    \caption{Representative examples from the SBM and resulting mask expectations (darker grid indicates edge probability closer to 1). (a) The expected mask is dense if all nodes and clusters are collapsed within a small region. (b) Clear-cut groups in the embedding space induce a block-diagonal mask. (c) Clusters located far apart from nodes lead to sparse masks.}
     \vspace{-4mm}
     \label{fig:sbm_sparsity}
\end{figure}

\paragraph{Random Edge Exploration.} While this approach enables backpropagation in the same $\calO(m)$ cost as in the forward step, this comes at the expense of not being able to propagate information through edges that were not sampled. This can be problematic when an edge probability accidentally collapses to zero, after which the edge becomes unlikely to ever be sampled even when it may be useful for the prediction task at hand. Therefore, we add a small perturbation $\delta > 0$ to each edge probability $p_{ij}$, allowing the model to explore new edges and resuscitate their sampling probabilities if necessary. 
We find that a $\delta$ as small as $0.01$ significantly helps in practice, and thus use this edge exploration scheme during training for our experiments.

\paragraph{Wouldn't the model always prefer full attention?} 
Note that the gradient $\partial\loss/\partial p_{ij}$ can be positive, which suppresses the probability of edge $(i,j)$. At first, it may seem counter-intuitive why the model would ever limit itself to using fewer edges during training without any sparsity-based regularizations. One explanation is that masked attention provides an easy way to reduce attention scores under finite head dimensions.
Under full attention, it is known that the representational space of attention score matrices is limited by the head dimension and softmax activation~\cite{bhojanapalli2020}.
This limitation inevitably introduces unwanted noise in the attention scores especially when working with long sequences.
In \modelname{}, however, the structural sparsity in masked attention introduces another dimension that induces a larger space of row-stochastic matrices (full attention is a special case of masked attention where $\BM_{ij}=1$ for all $i,j$). Therefore, it is reasonable that the model may encourage sparsity to leverage the additional expressiveness assuming the loss landscape has local optima within the sparse attention regime. Our experiments on the LRA benchmark show that this is indeed the case, as our \modelname{} converges to an average attention sparsity of 20\% to 30\% while outperforming Transformer with full attention.
We also show in the experiment that we can easily incorporate additional regularization that further encourages sparse attention masks.
\subsection{\modelname{} is a Universal Approximator}
\cutsubsectiondown

Leveraging previous work on the theoretical expressiveness of sparse attention~\cite{yun2020,bigbird}, we show that SBM-Transformer with a small modification\footnote{Here we consider a variant of SBM-Transformer where self-loops are added manually (i.e. $\BM_{ii}=1$ for all $i$). While this is useful in theoretical analysis, we find that not having self-loops slightly helps in empirical performance and hence omit self-loops for the main experiments.} retains the same level of expressibility as full attention. Specifically, we show that the low-rank structure of the underlying SBMs does not degrade the expressive power of Transformer, and that SBM-Transformer can universally approximate arbitrary functions with $\mathcal{O}(n)$ connections. For brevity, we provide a rough overview of the proof and defer further details to Appendix A.

\begin{theorem}\label{thm:expressive}
Let $f \in \mathcal{F}$ be class of continuous sequence-to-sequence functions. $\mathcal{T}^{h,r,m}_{SBM}$ denote the class of SBM-Transformers with $h$ attention heads, $m$ head dimension, and $r$ dimensions in hidden layers. Then for any $\epsilon > 0$ and $1 \leq p < \infty$, there exists a function $g \in \mathcal{T}^{h,m,r}_{SBM}$ such that
\begin{gather}
    \int_\mathbb{D} \|f(\BX) - \mathbb{E}[g(\BX)]\|_p^p d\BX \leq \epsilon 
\end{gather}
\end{theorem}

According to the main theorem of Yun~et~al.~(2020)~\cite{yun2019}, SBM-Transformer achieves universal approximability if 1) each node attends to itself, 2) the aggregation of all attention patterns contains a Hamiltonian path, and 3) there exists a path between all node pairs. While the first condition is trivially true due to our modification, the other two conditions require careful choice of three SBMs. Here we first parameterize one SBM to hard-assign tokens into $k$ equally-sized clusters, inducing a block-diagonal attention pattern. The other two SBMs are parameterized such that the two graphs together form a star graph with $k$ global relay tokens. Combining the three attention patterns lead to a parameterization of SBM-Transformer that satisfies all three conditions, hence proving the theorem.

\cutsubsectionup
\section{Experiments}\label{sec:experiments}
\cutsubsectiondown
\begin{figure}[!t]
     \centering
     \begin{subfigure}{0.561\textwidth}
         \centering
         \includegraphics[trim={0mm 0mm 2mm 2mm},clip,width=\textwidth]{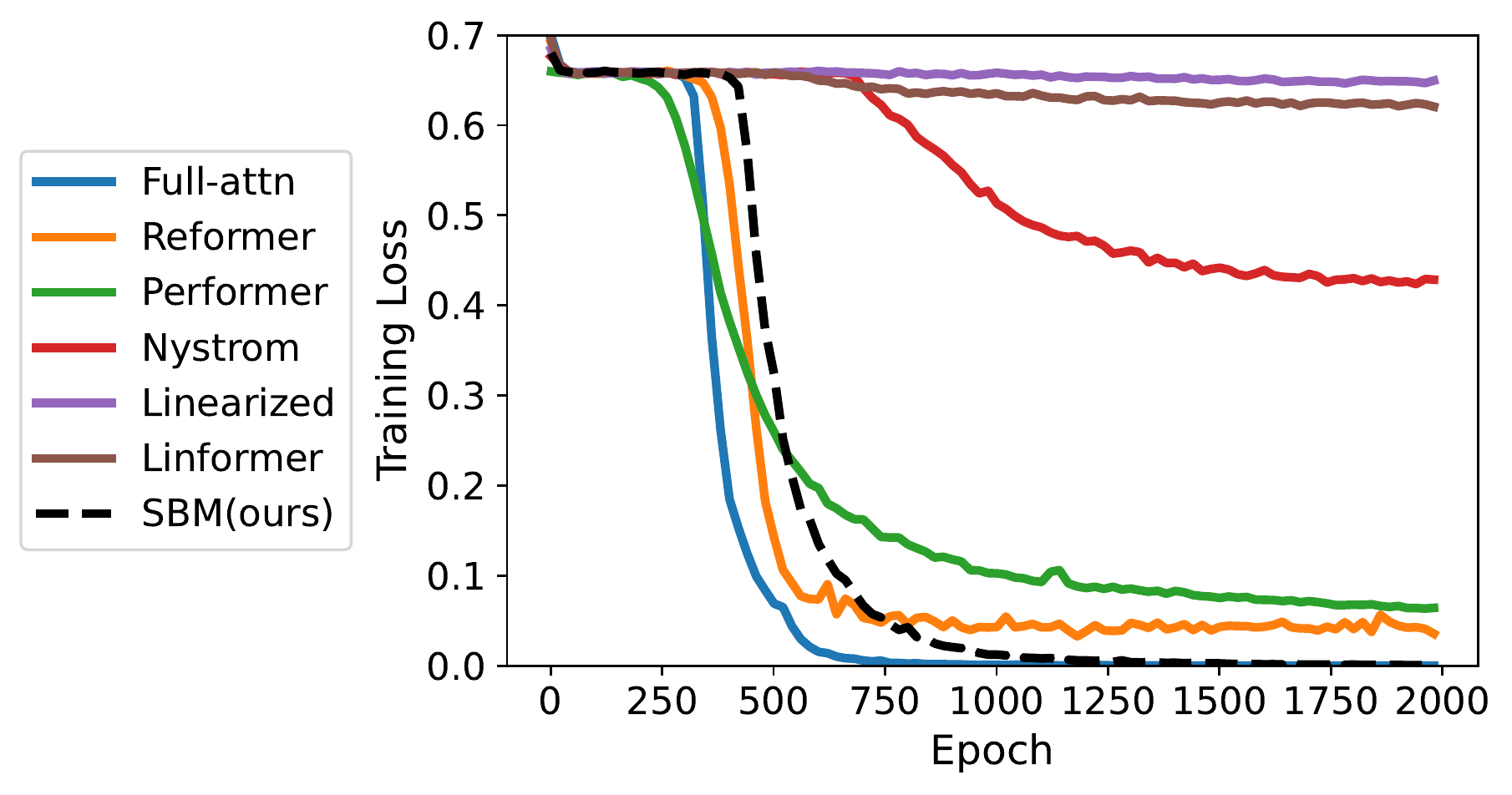}
     \end{subfigure}
     \hfill
     \begin{subfigure}{0.43\textwidth}
         \centering
         \includegraphics[trim={0mm 0mm 2mm 2mm},clip,width=\textwidth]{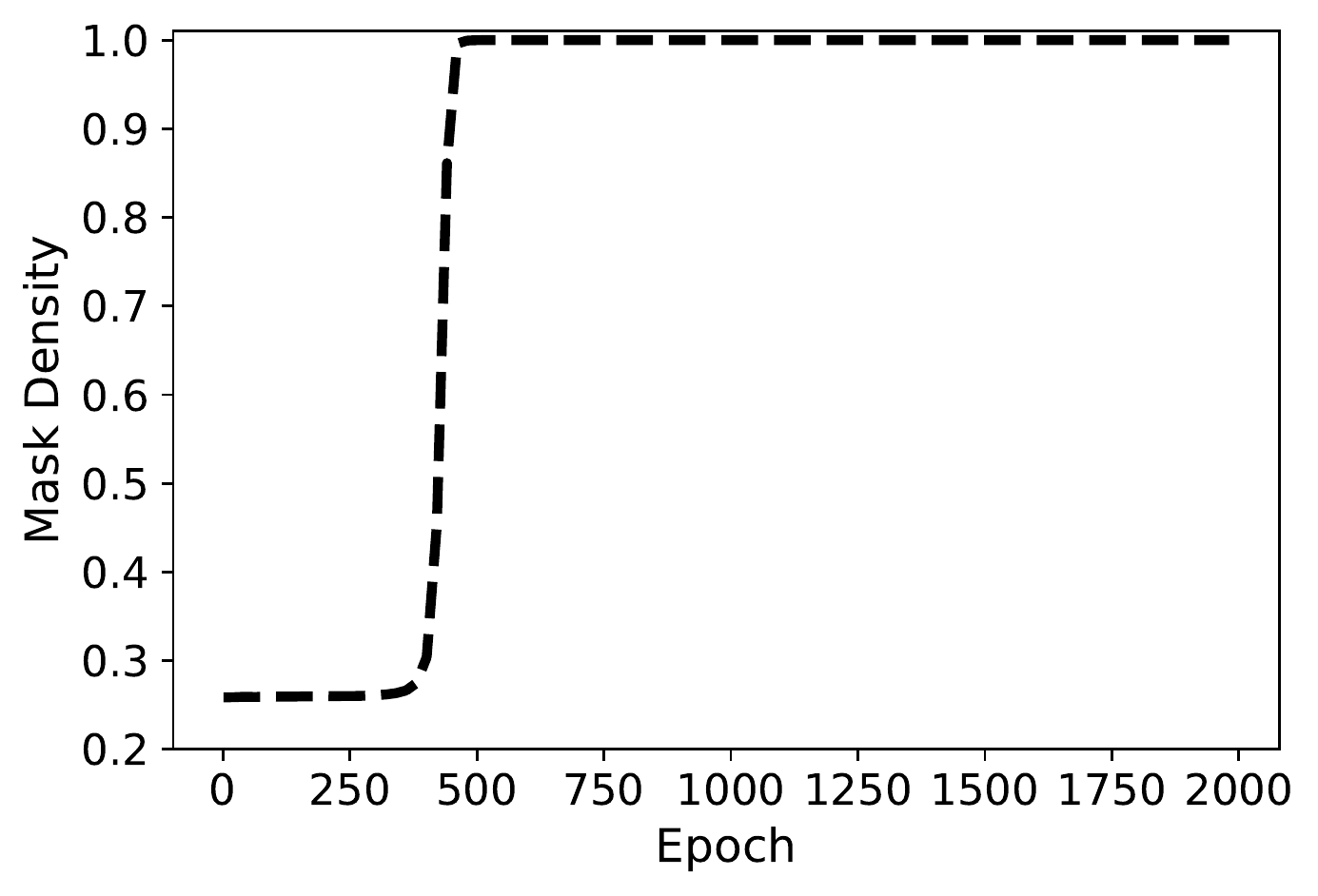}
     \end{subfigure}
     \vspace{-2mm}
     \caption{Loss (left) and mask density (right) of \modelname{} during training on the synthetic task. \modelname{} successfully converges to zero loss by tuning itself towards full attention.}
     \label{fig:synthetic_result}
\end{figure}

For empirical evaluations, we first use a synthetic task to show that our model is flexible enough to learn towards full attention when needed in contrast to previous works. We then experiment on Long Range Arena (LRA)~\cite{lra}, a benchmark widely used to assess the capacity of efficient Transformers in learning long-range contexts across different modalities. Lastly, we show results on the GLUE benchmark~\cite{glue} to assess the performance of SBM-Transformer in a downstream NLP setting. All experiments were run on a remote GCP server equipped with 16 NVIDIA A100 Tensor Core GPUs.

\subsection{Synthetic Task: Finding Repeated Tokens}\label{sec:synthetic}

\paragraph{Dataset.} We formulate a token-level binary classification task as follows: each input sequence consists of $N$ integers, each of which is uniformly sampled from $\{1,2,\dots,N\}$. We use $N=256$ in our setup. The prediction target is a sequence of equal length, where each token is labeled 1 if there exists a duplicate somewhere within the sequence, and 0 otherwise. Below is a simple example with $N=8$ that illustrates the task. We measure the performance of models via binary cross-entropy loss.
\begin{center}
    Input: \texttt{1 4 3 7 3 2 3 1} $\Rightarrow$ Target: \texttt{1 0 1 0 1 0 1 1}
\end{center}

\paragraph{Methods.} For this task, we compare \modelname{} with $k=128$ clusters against various efficient Transformers: Linear Transformer \cite{lineartransformer}, Linformer \cite{linformer}, Reformer \cite{reformer}, Performer \cite{performer}, and Nystr\"omformer \cite{nystromformer}. Across all methods, we use a single-layer and single-head architecture with 32 hidden dimensions. Note that due to this constrained setting, the sole head must perform full attention to compare each token to all the others in order to attain 100\% accuracy. All models are trained for 2000 epochs where a new batch of sequences is sampled on-the-fly at each epoch. We use a batch size of 256 and learning rate of 1e-3.

\paragraph{Results.} Figure~\ref{fig:synthetic_result} shows the training loss curves of each baseline method as well as \modelname{}. Full attention quickly converges to 100\% accuracy, which is expected as it computes all possible pairwise interactions by default. Other models that apply low-rank or kernelized attention fail to achieve the same level of accuracy, due to limited expressibility under the constrained setting. Though \modelname{} converges more slowly compared to full-attention, it demonstrates the ability to drive itself towards full-attention, eventually attaining zero loss.

\cutsubsectionup
\subsection{Long Range Arena (LRA)}
\cutsubsectiondown

To demonstrate that the flexible inductive bias of \modelname{} is effective for modeling long-range dependencies, we test \modelname{} against previous work on the LRA benchmark. We also test how the performance is affected with respect to applying a sparsity-based regularizer.

\paragraph{Dataset.} LRA~\cite{lra} consists of five different testbeds with varying modalities: \textsc{ListOps}~\cite{listops} is a 10-way classification task to map a sequence of single-digit numbers and 4 different set operations, to its corresponding solution. \textsc{Text}~\cite{text} is a binary classification task where byte-level IMDB movie reviews must be classified into one of positive or negative sentiments. \textsc{Retrieval}~\cite{retrieval} is also a char-level binary classification task, where two sequences from ACL Anthology papers are given as input, and the model must predict whether there exists a citation link between them. \textsc{Image}~\cite{image} is a 10-way classification task mapping flattened pixel-sequences from CIFAR-10 to its class. \textsc{Pathfinder}~\cite{pathfinder} provides flattened pixel-sequences from an image and the model must decide whether two circles in the image are connected by a dashed line. For this benchmark, we use the PyTorch implementation of LRA provided by the authors of Nystr\"omformer~\cite{nystromformer} and adhere to the same train-test splits. Performance in all five tasks is measured using classification accuracy.

\paragraph{Methods.} We compare \modelname{} against the same baselines as with the synthetic task above. For fair comparison, we set all Transformer models to use the default setting used in~\cite{nystromformer}, which fixes 2 layers, 2 attention heads, and 64 embedding dimensions. For \modelname{}, we use $k=128$ clusters. The output token representations are mean-pooled to obtain the sequence representation for all tasks. More details on the architecture setups can be found in Appendix C.

\paragraph{Results.} Table~\ref{tab:lra_results} shows the test accuracies of each method. Our \modelname{} achieves the best overall performance, ranking first in two tasks, and second in one other. \modelname{} also outperforms full attention in all five tasks while computing 30\% or less attention scores on average, which supports our claim that masked attention with partial attention score computations can be preferred over full attention depending on the task. With respect to the attention mask structure, we find that flexibility of SBM is indeed beneficial, as Reformer struggles in \textsc{ListOps}, most likely due to the inability of block-diagonal masks to model hierarchical contexts. 

\begin{table}[t!]
    \centering
    \resizebox{\textwidth}{!}{\begin{tabular}{c|ccccc|c}
        \toprule
        Model & \textsc{ListOps}(2K) & \textsc{Text}(3K) & \textsc{Retrieval}(4K) & \textsc{Image}(1K) & \textsc{Pathfinder}(1K) & Avg. \\
        \midrule
        Full-attention~\cite{transformer} & 37.22 & 64.93 & 79.55 & 40.38 & 74.26 & 59.27\\ 
        \midrule
        Linearized~\cite{lineartransformer} & \underline{37.46} & 64.90 & \bf 81.10 & 38.48 & \underline{74.61} & 59.31\\ 
        Reformer~\cite{reformer} & 22.92 & 64.70 & 77.25 & \bf 43.65 & 70.28 & 55.76\\ 
        Performer~\cite{performer} & 18.25 & 65.00 & 79.01 & 39.80 & 70.79 & 54.57\\ 
        Linformer~\cite{linformer} & \bf 38.44 & 56.28 & 78.09 & 39.53 & 67.62 & 55.99\\ 
        Nystr\"omformer~\cite{nystromformer} & 37.22 & \underline{65.46} & 79.35 & \underline{43.07} & 71.97 & \underline{59.41}\\ 
        \midrule
        \modelname{} (ours) & 37.45 (20.09\%) & \bf 65.79 (26.10\%) & \underline{80.00 (29.46\%)} & 41.31 (20.49\%) & \bf 75.12 (18.56\%) & \bf 59.93\\ 
        \bottomrule
    \end{tabular}
    }
    \vspace{2mm}
    \caption{LRA benchmark results. The sequence lengths are shown next to each task. For \modelname{}, we report the average attention sparsity across all layers and heads during test time in parentheses. Bold and underlined results indicate best and 2nd best test accuracy for each task.}
    \label{tab:lra_results}
\end{table}

\begin{table}[t!]
    \centering
    \resizebox{\textwidth}{!}{\begin{tabular}{c|ccccc|c}
        \toprule
        $\lambda$ & \textsc{ListOps}(2K) & \textsc{Text}(3K) & \textsc{Retrieval}(4K) & \textsc{Image}(1K) & \textsc{Pathfinder}(1K) & Avg. \\
        \midrule
        0 & 37.45 (20.09\%) & \bf 65.79 (26.10\%) & 80.00 (29.46\%) & 41.31 (20.49\%) & 75.12 (18.56\%) & 59.93\\
        $10^{-4}$ & 37.76 (10.48\%) & 65.48 (26.26\%) & 79.93 (24.62\%) & 41.35 (10.70\%) & \bf 75.46 (5.16\%) & \bf 60.00\\ 
        $10^{-3}$ & \bf 38.23 (10.46\%) & 65.18 (26.03\%) & 80.00 (21.70\%) & 41.17 (24.60\%) & 74.49 (3.82\%) & 59.81\\ 
        $10^{-2}$ & 38.20 (2.95\%) & 65.59 (22.43\%) & \bf 80.44 (6.99\%) & \bf 42.20 (3.95\%) & 72.79 (3.76\%) & 59.84\\ 
        $10^{-1}$ & 37.76 (1.15\%) & 64.48 (10.62\%) & 79.46 (2.49\%) & 41.35 (1.33\%) & 73.79 (2.61\%) & 59.37\\ 
        \bottomrule
    \end{tabular}
    }
    \vspace{2mm}
    \caption{LRA results of \modelname{} with increasing sparsity regularization weight $\lambda$. Bold results indicate best accuracy for each task and percentage in parentheses indicate average attention density. Sparsity regularization helps in reducing computational cost with small drop in performance.}
    \vspace{-3mm}
    \label{tab:lra_lambda}
\end{table}

\paragraph{Mask Density Regularization.}
To test if the model can effectively learn under a constraint on the computational cost, we also test the model under a sparsity-based regularizer that discourages excessive use of query-key edges. We penalize each sampled edge by adding to the predictive loss a weighted regularization term $\lambda\loss_s$, where $\loss_s$ denotes the average mask density across all attention heads. Table~\ref{tab:lra_lambda} shows the performance of \modelname{} across varying regularization weights. Under strong regularization, the model surprisingly retains competitive performance while significantly reducing the average mask density. 
This indicates that similar local optima are shared across regimes with varying attention density in the loss landscape, and the regularization term is able to drive the model towards finding optimal attention scores with smaller density.

\begin{table}[t!]
    \centering
    \resizebox{\textwidth}{!}{\begin{tabular}{c|ccccc|ccccc}
        \toprule
         & \multicolumn{5}{c}{Relative FLOP Count} & \multicolumn{5}{|c}{Relative Peak Memory Usage}\\
        Model & \textsc{L}(2K) & \textsc{T}(3K) & \textsc{R}(4K) & \textsc{I}(1K) & \textsc{P}(1K) & \textsc{L}(2K) & \textsc{T}(3K) & \textsc{R}(4K) & \textsc{I}(1K) & \textsc{P}(1K) \\
        \midrule
        Full-attention~\cite{transformer} & 1.00 & 1.00 & 1.00 & 1.00 & 1.00 & 1.00 & 1.00 & 1.00 & 1.00 & 1.00 \\ 
        \midrule
        Linearized~\cite{lineartransformer} & 0.02 & 0.01 & 0.02 & 0.04 & 0.04 & 0.18 & 0.16 & 0.12 & 0.42 & 0.42\\ 
        Reformer~\cite{reformer} & 0.05 & 0.03 & 0.05 & 0.10 & 0.10 & 0.39 & 0.31 & 0.18 & 0.72 & 0.72\\  
        Performer~\cite{performer} & 0.18 & 0.12 & 0.18 & 0.36 & 0.36 & 0.76 & 0.70 & 0.60 & 0.96 & 0.96\\ 
        Linformer~\cite{linformer} & 0.33 & 0.22 & 0.33 & 0.66 & 0.66 & 0.26 & 0.22 & 0.14 & 0.34 & 0.34\\ 
        Nystr\"omformer~\cite{nystromformer} & 1.09 & 0.70 & 1.09 & 2.37 & 2.37 & 0.34 & 0.27 & 0.16 & 0.70 & 0.70\\
        \midrule
        \modelname{} (ours) & 0.07 & 0.23 & 0.08 & 0.27 & 0.29 & 0.19 & 1.01 & 0.19 & 0.39 & 0.48\\  
        \bottomrule
    \end{tabular}
    }
    \vspace{2mm}
    \caption{Per-example relative FLOP count and peak memory usage during LRA inference.}
    \label{tab:lra_cost}
\end{table}

\begin{figure}[!t]
    \vspace{-3mm}
    \begin{center}
        \includegraphics[width=0.85\linewidth]{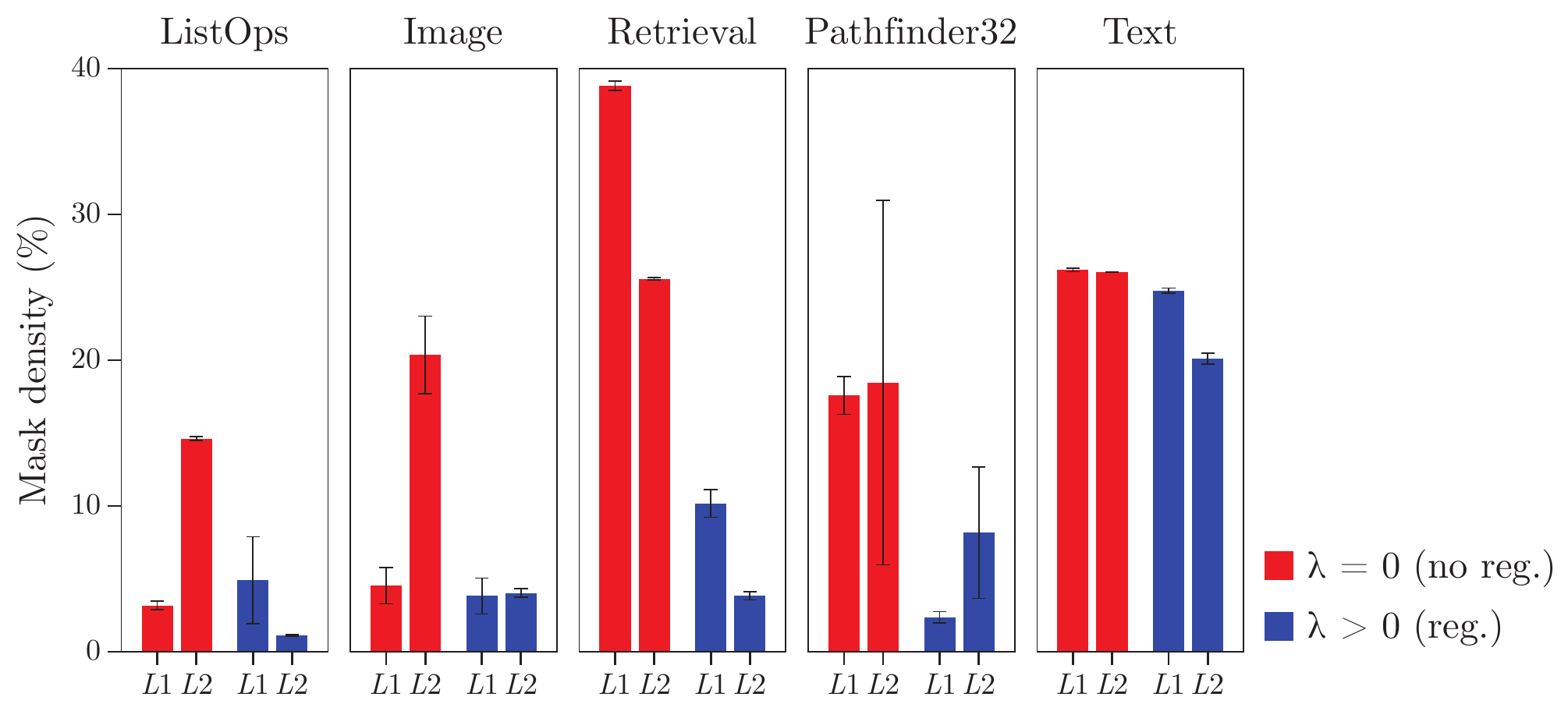}
    \end{center}
    \vspace{-3mm}
    \caption{Average and standard deviation of density of masks sampled across the test set for each LRA task. The $x$-axis indicates the lower (L1) and upper (L2) layers and each bar represents the density averaged between the two attention heads in each layer.
    }
     \label{fig:diversity}
\end{figure}

\paragraph{Efficiency.}
Furthermore, we compare computational costs during inference by measuring FLOP count and peak memory usage. For SBM-Transformer, we test the model trained under $\lambda=10^{-1}$. Due to lack of support for sparse tensor operations in existing FLOP-counters, we measure FLOP counts by manually enumerating through each tensor operation. Table~\ref{tab:lra_cost} shows that SBM-Transformer is comparably efficient across all tasks except for \textsc{Text}, where SBM-Transformer showed the largest average mask density. Note that while the cost of other baselines are fixed after initialization, the cost of SBM-Transformer is data-adaptive and can vary input-by-input. Further analysis and qualitative examples demonstrating the input-dependent attention mask densities can be found in Appendix C.

\paragraph{Layerwise Diversity in Sparsity.}

We also compare the densities of masks sampled at each layer of \modelname{} during test time to examine whether our model is capable of diversifying sparsity across layers for better performance. Recall that this allows models to gather information in different levels, as seen in pretrained BERT where lower layers focus on the overall content via dense attention while upper layers gather syntactic information with tree-like patterns~\cite{clark2019}. For each of the five tasks, we pick two highest-performing models (one for unregularized and another for regularized) for measurement. Figure~\ref{fig:diversity} shows the average layer-wise mask densities of unregularized and regularized \modelname{}s across different tasks. We find that under no regularization, the two layers can differ by more than 10\% in tasks such as \textsc{ListOps} and \textsc{Image}. This may be due to the hierarchical and compositional structure of the two tasks. We also find that the variation is relatively low in \textsc{Text} with densities around 25\%, indicating that the task requires broad attention overall. Lastly, the standard deviation is extremely large in upper layers for \textsc{Pathfinder}, showing that it samples a wide variety of masks depending on the input.

\cutsubsectionup
\subsection{General Language Understanding Evaluation (GLUE)}
\cutsubsectiondown
To check whether its strong performance demonstrated in LRA extends to the downstream NLP setting as well, we evaluate SBM-Transformer against baselines on the GLUE benchmark~\cite{glue}.

\cutparagraphup
\paragraph{Dataset.} 
We consider four NLP tasks in GLUE~\cite{glue}. \textsc{SST-2}~\cite{sst-2} consists of movie reviews the model must predict their positive or negative sentiments. For \textsc{QQP}~\cite{qqp}, the task is to determine whether one question is a paraphrase of the other given a pair of questions. \textsc{MNLI}~\cite{mnli} consists of sentence pairs, each with a target label indicating whether the two sentences are connected through entailment, contradiction, or neither. \textsc{QNLI}~\cite{qnli} consists of sentence-question pairs and the task is to determine whether the sentence contains an answer to the question. Each task is formulated as sequence classification, and we measure performance by F1 score on the respective validation sets.

\cutparagraphup
\paragraph{Methods.} 
Following previous work~\cite{nystromformer}, we arrange a small variant of BERT~\cite{bert} with 4 layers, 8 attention heads, and 512 embedding dimensions. We replace full attention with each attention module used in previous experiments. For SBM-Transformer, we use $k=128$ clusters without sparsity regularization (i.e. $\lambda = 0$). Here, we find that adding local attention significantly boosts performance, and thus fix a sliding window of size 64 to SBM-Transformer. We first pretrain each model under the masked language modeling objective for 50 epochs on a corpus with text from English Wikipedia, BookCorpus~\cite{bookcorpus}, and RealNews~\cite{realnews}. We then finetune each pretrained model for 5 epochs on the GLUE training sets. More details on the architecture and training setup can be found in Appendix C.

\begin{wraptable}{r}{0.5\textwidth}
    \vspace{-4mm}
    \centering
    \resizebox{0.5\textwidth}{!}{\begin{tabular}{c|cccc}
        \toprule
        Model & \textsc{SST-2} & \textsc{QQP} & \textsc{MNLI} & \textsc{QNLI}\\
        \midrule
        Full-attention~\cite{transformer} & \textbf{89.8} & 84.7 & 84.0 & \textbf{85.0}\\ 
        \midrule
        Reformer~\cite{reformer} & 89.3 & 84.4 & 83.9 & 84.0 \\ 
        Performer~\cite{performer} & 82.0 & 65.6 & 71.4 & 59.3\\ 
        Linformer~\cite{linformer} & 82.0 & 83.2 & 79.3 & 82.5\\ 
        Nystr\"omformer~\cite{nystromformer} & 89.7 & 83.2 & \textbf{84.1} & 84.9\\ 
        \midrule
        \modelname{} (ours) & \textbf{89.8} & \textbf{85.2} & 83.5 & 83.6\\ 
        \bottomrule
    \end{tabular}
    }
    \caption{GLUE benchmark results. Bold results indicate best accuracy for each task.}
    \label{tab:glue_results}
    \vspace{-5mm}
\end{wraptable}

\paragraph{Results.}
Table~\ref{tab:glue_results} reports the F1 scores of each method on different NLP tasks. SBM-Transformer performs competitively against full attention overall, and outperforms all baselines in \textsc{SST-2} and \textsc{QQP}. We also find that the fine-tuned SBM-Transformer models use 13.5\% dense attention masks on average across all tasks, showing that the model can encode useful information from input sentences effectively under highly sparse attention. 

\section{Conclusion}\label{sec:conclusion}
We propose \modelname{}, an efficient Transformer that can data-adaptively choose its attention sparsity between sparse and full attention without the need to explicitly compute the full attention score matrix. Theoretically, we show that our model enjoys the same expressibility as the original Transformer due to the flexibility of the latent SBM. Empirical experiments on LRA and GLUE show that our model performs competitively against previous state-of-the-art efficient Transformers.

Nonetheless, there are limitations due to sparse tensor operations being less optimized on GPU kernels. In the LRA experiments, we found that SBM-Transformer can result in longer runtimes compared to dense counterparts while its memory usage is much lower. While previous sparsity-based attention mechanisms with block-sparse attention are much more amenable for GPU computation~\cite{bigbird,sparsetransformer,longformer}, our work requires an architecture with better workload balancing and acceleration under unstructured sparsity, for which there is ongoing work~\cite{sparsert,latency-aware}.

We still believe this work is valuable as it is the first approach to induce per-example attention sparsity, allowing the model to adjust its computational cost based on the input. The cost being dependent on the number of edges also allows practitioners to easily impose constraints based on the available computational resources.
We hope to see more GPU-friendly tensor operations optimized for fine-grained sparsity in the future, at which point the value of this work will increase even further. As we propose a foundational replacement for the scaled dot-product attention module in the Transformer architecture, we do not expect any immediate negative societal impact due to this work.



\begin{ack}
We would like to thank Kun Dong for the insightful comments. This work was supported by Institute of Information \& communications
Technology Planning \& Evaluation (IITP) grant funded by the Korea government (MSIT) (No. 2022-0-00926, 2022-0-00959, 2021-0-02068, and 2019-0-00075).
\end{ack}

\bibliographystyle{abbrv}
\bibliography{neurips_2022}

\begin{thebibliography}{10}

\bibitem{abbe2017community}
E.~Abbe.
\newblock Community detection and stochastic block models: recent developments.
\newblock {\em The Journal of Machine Learning Research}, 18(1):6446--6531,
  2017.

\bibitem{mmsbm}
E.~M. Airoldi, D.~Blei, S.~Fienberg, and E.~Xing.
\newblock Mixed membership stochastic blockmodels.
\newblock {\em Advances in neural information processing systems}, 21, 2008.

\bibitem{longformer}
I.~Beltagy, M.~E. Peters, and A.~Cohan.
\newblock Longformer: The long-document transformer.
\newblock {\em CoRR}, abs/2004.05150, 2020.

\bibitem{ste}
Y.~Bengio, N.~L{\'{e}}onard, and A.~C. Courville.
\newblock Estimating or propagating gradients through stochastic neurons for
  conditional computation.
\newblock {\em CoRR}, abs/1308.3432, 2013.

\bibitem{bhojanapalli2020}
S.~Bhojanapalli, C.~Yun, A.~S. Rawat, S.~J. Reddi, and S.~Kumar.
\newblock Low-rank bottleneck in multi-head attention models.
\newblock {\em CoRR}, abs/2002.07028, 2020.

\bibitem{bollobas1998random}
B.~Bollob{\'a}s.
\newblock Random graphs.
\newblock In {\em Modern graph theory}, pages 215--252. Springer, 1998.

\bibitem{gpt-3}
T.~B. Brown, B.~Mann, N.~Ryder, M.~Subbiah, J.~Kaplan, P.~Dhariwal,
  A.~Neelakantan, P.~Shyam, G.~Sastry, A.~Askell, S.~Agarwal,
  A.~Herbert{-}Voss, G.~Krueger, T.~Henighan, R.~Child, A.~Ramesh, D.~M.
  Ziegler, J.~Wu, C.~Winter, C.~Hesse, M.~Chen, E.~Sigler, M.~Litwin, S.~Gray,
  B.~Chess, J.~Clark, C.~Berner, S.~McCandlish, A.~Radford, I.~Sutskever, and
  D.~Amodei.
\newblock Language models are few-shot learners.
\newblock {\em CoRR}, abs/2005.14165, 2020.

\bibitem{qqp}
Z.~Chen, H.~Zhang, X.~Zhang, and L.~Zhao.
\newblock Quora question pairs.
\newblock 2017.

\bibitem{sparsetransformer}
R.~Child, S.~Gray, A.~Radford, and I.~Sutskever.
\newblock Generating long sequences with sparse transformers.
\newblock {\em CoRR}, abs/1904.10509, 2019.

\bibitem{performer}
K.~Choromanski, V.~Likhosherstov, D.~Dohan, X.~Song, A.~Gane, T.~Sarl{\'{o}}s,
  P.~Hawkins, J.~Davis, A.~Mohiuddin, L.~Kaiser, D.~Belanger, L.~J. Colwell,
  and A.~Weller.
\newblock Rethinking attention with performers.
\newblock {\em CoRR}, abs/2009.14794, 2020.

\bibitem{clark2019}
K.~Clark, U.~Khandelwal, O.~Levy, and C.~D. Manning.
\newblock What does {BERT} look at? an analysis of bert's attention.
\newblock {\em CoRR}, abs/1906.04341, 2019.

\bibitem{correia2019}
G.~M. Correia, V.~Niculae, and A.~F.~T. Martins.
\newblock Adaptively sparse transformers.
\newblock {\em CoRR}, abs/1909.00015, 2019.

\bibitem{smyrf}
G.~Daras, N.~Kitaev, A.~Odena, and A.~G. Dimakis.
\newblock {SMYRF:} efficient attention using asymmetric clustering.
\newblock {\em CoRR}, abs/2010.05315, 2020.

\bibitem{bert}
J.~Devlin, M.-W. Chang, K.~Lee, and K.~Toutanova.
\newblock {BERT}: Pre-training of deep bidirectional transformers for language
  understanding.
\newblock In {\em Proceedings of the 2019 Conference of the North {A}merican
  Chapter of the Association for Computational Linguistics: Human Language
  Technologies, Volume 1 (Long and Short Papers)}, pages 4171--4186,
  Minneapolis, Minnesota, June 2019. Association for Computational Linguistics.

\bibitem{ViT}
A.~Dosovitskiy, L.~Beyer, A.~Kolesnikov, D.~Weissenborn, X.~Zhai,
  T.~Unterthiner, M.~Dehghani, M.~Minderer, G.~Heigold, S.~Gelly, J.~Uszkoreit,
  and N.~Houlsby.
\newblock An image is worth 16x16 words: Transformers for image recognition at
  scale.
\newblock In {\em International Conference on Learning Representations}, 2021.

\bibitem{funke2019sbm}
T.~Funke and T.~Becker.
\newblock Stochastic block models: A comparison of variants and inference
  methods.
\newblock {\em PLOS ONE}, 14(4):1--40, 04 2019.

\bibitem{kaiming2015delving}
K.~He, X.~Zhang, S.~Ren, and J.~Sun.
\newblock Delving deep into rectifiers: Surpassing human-level performance on
  imagenet classification.
\newblock {\em CoRR}, abs/1502.01852, 2015.

\bibitem{LSTM}
S.~Hochreiter and J.~Schmidhuber.
\newblock Long short-term memory.
\newblock {\em Neural computation}, 9(8):1735--1780, 1997.

\bibitem{huang2021efficient}
L.~Huang, S.~Cao, N.~Parulian, H.~Ji, and L.~Wang.
\newblock Efficient attentions for long document summarization.
\newblock In {\em Proceedings of the 2021 Conference of the North American
  Chapter of the Association for Computational Linguistics: Human Language
  Technologies}, pages 1419--1436, 2021.

\bibitem{dcsbm}
B.~Karrer and M.~E. Newman.
\newblock Stochastic blockmodels and community structure in networks.
\newblock {\em Physical review E}, 83(1):016107, 2011.

\bibitem{lineartransformer}
A.~Katharopoulos, A.~Vyas, N.~Pappas, and F.~Fleuret.
\newblock Transformers are rnns: Fast autoregressive transformers with linear
  attention.
\newblock {\em CoRR}, abs/2006.16236, 2020.

\bibitem{reformer}
N.~Kitaev, L.~Kaiser, and A.~Levskaya.
\newblock Reformer: The efficient transformer.
\newblock {\em CoRR}, abs/2001.04451, 2020.

\bibitem{image}
A.~Krizhevsky.
\newblock Learning multiple layers of features from tiny images.
\newblock 2009.

\bibitem{osbm}
P.~Latouche, E.~Birmel{\'e}, and C.~Ambroise.
\newblock Overlapping stochastic block models with application to the french
  political blogosphere.
\newblock {\em The Annals of Applied Statistics}, 5(1):309--336, 2011.

\bibitem{pathfinder}
D.~Linsley, J.~Kim, V.~Veerabadran, and T.~Serre.
\newblock Learning long-range spatial dependencies with horizontal
  gated-recurrent units.
\newblock {\em CoRR}, abs/1805.08315, 2018.

\bibitem{text}
A.~L. Maas, R.~E. Daly, P.~T. Pham, D.~Huang, A.~Y. Ng, and C.~Potts.
\newblock Learning word vectors for sentiment analysis.
\newblock In {\em Proceedings of the 49th Annual Meeting of the Association for
  Computational Linguistics: Human Language Technologies}, pages 142--150,
  Portland, Oregon, USA, June 2011. Association for Computational Linguistics.

\bibitem{martins2020}
A.~F.~T. Martins, M.~V. Treviso, A.~Farinhas, V.~Niculae, M.~A.~T. Figueiredo,
  and P.~M.~Q. Aguiar.
\newblock Sparse and continuous attention mechanisms.
\newblock {\em CoRR}, abs/2006.07214, 2020.

\bibitem{listops}
N.~Nangia and S.~R. Bowman.
\newblock Listops: {A} diagnostic dataset for latent tree learning.
\newblock {\em CoRR}, abs/1804.06028, 2018.

\bibitem{narang2021}
S.~Narang, H.~W. Chung, Y.~Tay, W.~Fedus, T.~F{\'{e}}vry, M.~Matena, K.~Malkan,
  N.~Fiedel, N.~Shazeer, Z.~Lan, Y.~Zhou, W.~Li, N.~Ding, J.~Marcus,
  A.~Roberts, and C.~Raffel.
\newblock Do transformer modifications transfer across implementations and
  applications?
\newblock {\em CoRR}, abs/2102.11972, 2021.

\bibitem{Transformer-NMT}
M.~Ott, S.~Edunov, D.~Grangier, and M.~Auli.
\newblock Scaling neural machine translation.
\newblock In {\em Proceedings of the Third Conference on Machine Translation:
  Research Papers}, pages 1--9, Brussels, Belgium, Oct. 2018. Association for
  Computational Linguistics.

\bibitem{hsbm}
T.~P. Peixoto.
\newblock Hierarchical block structures and high-resolution model selection in
  large networks.
\newblock {\em Physical Review X}, 4(1):011047, 2014.

\bibitem{retrieval}
D.~R. Radev, P.~Muthukrishnan, and V.~Qazvinian.
\newblock The {ACL} {A}nthology network.
\newblock In {\em Proceedings of the 2009 Workshop on Text and Citation
  Analysis for Scholarly Digital Libraries ({NLPIR}4{DL})}, pages 54--61,
  Suntec City, Singapore, Aug. 2009. Association for Computational Linguistics.

\bibitem{qnli}
P.~Rajpurkar, J.~Zhang, K.~Lopyrev, and P.~Liang.
\newblock {SQ}u{AD}: 100,000+ questions for machine comprehension of text.
\newblock In {\em Proceedings of the 2016 Conference on Empirical Methods in
  Natural Language Processing}, pages 2383--2392, Austin, Texas, Nov. 2016.
  Association for Computational Linguistics.

\bibitem{MSATranformer}
R.~M. Rao, J.~Liu, R.~Verkuil, J.~Meier, J.~Canny, P.~Abbeel, T.~Sercu, and
  A.~Rives.
\newblock Msa transformer.
\newblock In M.~Meila and T.~Zhang, editors, {\em Proceedings of the 38th
  International Conference on Machine Learning}, volume 139 of {\em Proceedings
  of Machine Learning Research}, pages 8844--8856. PMLR, 18--24 Jul 2021.

\bibitem{fastrg}
K.~Rohe, J.~Tao, X.~Han, and N.~Binkiewicz.
\newblock A note on quickly sampling a sparse matrix with low rank expectation.
\newblock {\em The Journal of Machine Learning Research}, 19(1):3040--3052,
  2018.

\bibitem{sparsebert}
H.~Shi, J.~Gao, X.~Ren, H.~Xu, X.~Liang, Z.~Li, and J.~T. Kwok.
\newblock Sparsebert: Rethinking the importance analysis in self-attention.
\newblock {\em CoRR}, abs/2102.12871, 2021.

\bibitem{sst-2}
R.~Socher, A.~Perelygin, J.~Wu, J.~Chuang, C.~D. Manning, A.~Y. Ng, and
  C.~Potts.
\newblock Recursive deep models for semantic compositionality over a sentiment
  treebank.
\newblock In {\em Proceedings of the 2013 conference on empirical methods in
  natural language processing}, pages 1631--1642, 2013.

\bibitem{lra}
Y.~Tay, M.~Dehghani, S.~Abnar, Y.~Shen, D.~Bahri, P.~Pham, J.~Rao, L.~Yang,
  S.~Ruder, and D.~Metzler.
\newblock Long range arena: {A} benchmark for efficient transformers.
\newblock {\em CoRR}, abs/2011.04006, 2020.

\bibitem{etsurvey}
Y.~Tay, M.~Dehghani, D.~Bahri, and D.~Metzler.
\newblock Efficient transformers: {A} survey.
\newblock {\em CoRR}, abs/2009.06732, 2020.

\bibitem{transformer}
A.~Vaswani, N.~Shazeer, N.~Parmar, J.~Uszkoreit, L.~Jones, A.~N. Gomez,
  L.~Kaiser, and I.~Polosukhin.
\newblock Attention is all you need.
\newblock {\em CoRR}, abs/1706.03762, 2017.

\bibitem{gat}
P.~Veli{\v{c}}kovi{\'{c}}, G.~Cucurull, A.~Casanova, A.~Romero, P.~Li{\`{o}},
  and Y.~Bengio.
\newblock {Graph Attention Networks}.
\newblock {\em International Conference on Learning Representations}, 2018.
\newblock accepted as poster.

\bibitem{walker1977}
A.~J. Walker.
\newblock An efficient method for generating discrete random variables with
  general distributions.
\newblock {\em ACM Trans. Math. Softw.}, 3(3):253–256, sep 1977.

\bibitem{glue}
A.~Wang, A.~Singh, J.~Michael, F.~Hill, O.~Levy, and S.~R. Bowman.
\newblock {GLUE:} {A} multi-task benchmark and analysis platform for natural
  language understanding.
\newblock {\em CoRR}, abs/1804.07461, 2018.

\bibitem{wang2019dgl}
M.~Wang, D.~Zheng, Z.~Ye, Q.~Gan, M.~Li, X.~Song, J.~Zhou, C.~Ma, L.~Yu,
  Y.~Gai, T.~Xiao, T.~He, G.~Karypis, J.~Li, and Z.~Zhang.
\newblock Deep graph library: A graph-centric, highly-performant package for
  graph neural networks.
\newblock {\em arXiv preprint arXiv:1909.01315}, 2019.

\bibitem{linformer}
S.~Wang, B.~Z. Li, M.~Khabsa, H.~Fang, and H.~Ma.
\newblock Linformer: Self-attention with linear complexity.
\newblock {\em CoRR}, abs/2006.04768, 2020.

\bibitem{sparsert}
Z.~Wang.
\newblock Sparsert: Accelerating unstructured sparsity on gpus for deep
  learning inference.
\newblock {\em arXiv preprint arXiv:2008.11849}, 2020.

\bibitem{mnli}
A.~Williams, N.~Nangia, and S.~Bowman.
\newblock A broad-coverage challenge corpus for sentence understanding through
  inference.
\newblock In {\em Proceedings of the 2018 Conference of the North {A}merican
  Chapter of the Association for Computational Linguistics: Human Language
  Technologies, Volume 1 (Long Papers)}, pages 1112--1122, New Orleans,
  Louisiana, June 2018. Association for Computational Linguistics.

\bibitem{nystromformer}
Y.~Xiong, Z.~Zeng, R.~Chakraborty, M.~Tan, G.~Fung, Y.~Li, and V.~Singh.
\newblock Nystr{\"{o}}mformer: {A} nystr{\"{o}}m-based algorithm for
  approximating self-attention.
\newblock {\em CoRR}, abs/2102.03902, 2021.

\bibitem{yun2019}
C.~Yun, S.~Bhojanapalli, A.~S. Rawat, S.~J. Reddi, and S.~Kumar.
\newblock Are transformers universal approximators of sequence-to-sequence
  functions?
\newblock {\em CoRR}, abs/1912.10077, 2019.

\bibitem{yun2020}
C.~Yun, Y.~Chang, S.~Bhojanapalli, A.~S. Rawat, S.~J. Reddi, and S.~Kumar.
\newblock O(n) connections are expressive enough: Universal approximability of
  sparse transformers.
\newblock {\em CoRR}, abs/2006.04862, 2020.

\bibitem{bigbird}
M.~Zaheer, G.~Guruganesh, A.~Dubey, J.~Ainslie, C.~Alberti,
  S.~Onta{\~{n}}{\'{o}}n, P.~Pham, A.~Ravula, Q.~Wang, L.~Yang, and A.~Ahmed.
\newblock Big bird: Transformers for longer sequences.
\newblock {\em CoRR}, abs/2007.14062, 2020.

\bibitem{realnews}
R.~Zellers, A.~Holtzman, H.~Rashkin, Y.~Bisk, A.~Farhadi, F.~Roesner, and
  Y.~Choi.
\newblock Defending against neural fake news.
\newblock In H.~Wallach, H.~Larochelle, A.~Beygelzimer, F.~d\textquotesingle
  Alch\'{e}-Buc, E.~Fox, and R.~Garnett, editors, {\em Advances in Neural
  Information Processing Systems}, volume~32. Curran Associates, Inc., 2019.

\bibitem{zhang2021multi}
P.~Zhang, X.~Dai, J.~Yang, B.~Xiao, L.~Yuan, L.~Zhang, and J.~Gao.
\newblock Multi-scale vision longformer: A new vision transformer for
  high-resolution image encoding.
\newblock In {\em Proceedings of the IEEE/CVF International Conference on
  Computer Vision}, pages 2998--3008, 2021.

\bibitem{latency-aware}
M.~Zhu and Y.~Xie.
\newblock Taming unstructured sparsity on gpus via latency-aware optimization.
\newblock In {\em 2020 57th ACM/IEEE Design Automation Conference (DAC)}, pages
  1--6, 2020.

\bibitem{bookcorpus}
Y.~Zhu, R.~Kiros, R.~Zemel, R.~Salakhutdinov, R.~Urtasun, A.~Torralba, and
  S.~Fidler.
\newblock Aligning books and movies: Towards story-like visual explanations by
  watching movies and reading books.
\newblock In {\em Proceedings of the IEEE international conference on computer
  vision}, pages 19--27, 2015.

\end{thebibliography}


\newpage
\appendix
\section{Proof of Theorem 1}

Here we provide a detailed proof to show that \modelname{} is a universal approximator of arbitrary sequence-to-sequence functions. Note that a trivial solution is to use a dense mask $\BM$ equal to the all-one matrix with rank 1, in which case \modelname{} becomes equivalent to the full attention Transformer~\cite{transformer} that is already known to achieve universal approximability~\cite{yun2019}. Instead, we show that there also exists a solution with $\calO(n)$ connections, leveraging previous analyses under sparse attention by Yun~et~al.~(2020)~\cite{yun2020} and Zaheer~et~al.~(2020)~\cite{bigbird}.

For theoretical analysis, we consider a variant of \modelname{} that manually adds self-loops in the bipartite graph such that $\BM_{ii} = 1$ for all $i$. While adding in self-loops help towards analyzing expressibility, we find that it does not help empirically, and hence omit the modification in the our main method during experimentation. A comparison on performance on the LRA benchmark can be found below in Appendix~\ref{app:results}.

Here we restate the necessary conditions from~\cite{yun2020}. Let $\mathcal{A}_i^l \subseteq [n]$ denote the sparsity pattern of $i$-th token in the $l$ attention pattern: $j \in \mathcal{A}_i^l$ if query $i$ attends to key $j$ in the $l$-th pattern. Then, the main theorem of Yun~et~al.,(2020)~\cite{yun2020} states that as long as the set of $p$ sparsity patterns $\{\mathcal{A}_i^l\}_{l=1}^p$ and the probability mapping $\rho$ ({\it e.g.}, softmax) of the sparse Transformer model satisfy the two assumptions below, then model achieves universal approximability with finite number of layers.

\begin{assumption}\label{ass:pattern}
The sparsity patterns $\{\mathcal{A}_i^l\}$ satisfy the following:
\begin{enumerate}
    \item For all $i \in [n]$ and $l \in [p]$, we have $i \in \mathcal{A}_i^l$
    \item There exists a permutation $\gamma: [n] \to [n]$ such that, for all $i \in [n-1]$, $\gamma(i) \in \cup_{l=1}^p \mathcal{A}_{\gamma(i+1)}^l$.
    \item There exists a finite $s \in \mathbb{N}$ such that $s = \min\{u \;|\; \mathcal{S}_i^u = [n] \text{ for all } i \in [n]\}$ where $\mathcal{S}_i^u$ is defined recursively by $\mathcal{S}_i^1 \coloneqq \mathcal{A}_k^1$ and $\mathcal{S}_i^t \coloneqq \wbigcup_{j \in \mathcal{A}_i^{(t-1)\text{ mod }p+1}} \mathcal{S}_j^{t-1}$.
\end{enumerate}
\end{assumption}

\begin{assumption}\label{ass:softmax}
For any $\zeta > 0$ and $\eta \in (0,1]$, $\exists t>0$ such that, for any column input $\Bv$ satisfying $v_{j^*}-\max_{j\neq j^*} v_j \geq \zeta$ (where $j^* = \arg\max_j v_j$), we have $\rho[t\Bv]_{j^*} \geq 1 - \eta$ and $\sum_{j \neq j^*} \rho[t\Bv]_j \leq \eta$
\end{assumption}

When viewing each attention pattern $\mathcal{A}^l$ as a directed graph $G^l = (V, E^l)$ with node set $V\coloneqq [n]$ and edge set $E^l\coloneqq\{(j,i)|j\in\mathcal{A}^l_i \;\forall i,j\}$, each item in Assumption~\ref{ass:pattern} can be equivalently written as
\begin{enumerate}[label={\it Condition} \arabic*:,leftmargin=*,align=left]
    \item For all directed graphs $G^l$, each node has a self-loop.
    \item The aggregation of all $p$ graphs $G^* \coloneqq (V, \cup_{l=1}^p E^l)$ has a Hamiltonian path that spans all $n$ nodes.
    \item In a finite aggregation of $s$ graphs $G^{*s} \coloneqq (V, \cup_{l=1}^s E^l)$, there exists a path between all possible pairs of nodes.
\end{enumerate}

Because we use the same softmax probability mapping, which is already proven to satisfy Assumption~\ref{ass:softmax} in~\cite{yun2020}, we are left to show that there exists a parameterization of \modelname{} such that the expected attention mask patterns together satisfy the three conditions above.
To do so, we first show that a simple random ER-graph~\cite{bollobas1998random} can be expected to have at least one Hamiltonian cycle with expected number of edges linear in the sequence length.

\begin{lemma}\label{lem:hamiltonian}
Assume a directed Erd\H{o}s-R\'enyi random graph $G(n,p)$ where each directed edge exists with probability $p$. Then, for any number of nodes $n$, there exists a probability $p$ such that the expected number of edges is $\calO(n)$ and the expected number of Hamiltonian cycles in $G(n,p)$ is greater than or equal to 1.
\end{lemma}
\begin{proof}
We start the proof by formulating the expected number of Hamiltonian cycles in $G(n,p)$. Assuming directed edges, there exist $(n-1)!$ permutations, each of which represent different possible Hamiltonian cycles. Say we have $(n-1)!$ random variables $\{X_i\}_i^{(n-1)!}$ where each $X_i$ equals 1 when the corresponding Hamiltonian cycle exists in $G$, 0 otherwise. By linearity of expectation, the expected number of Hamiltonian cycles equals $\sum_{i=1}^{(n-1)!} \E[X_i]$. Then, note the probability of $X_i = 1$ equals $p^n$ for all $i$ since we require $n$ directed edges to exist for each cycle. Therefore, the total expected number of Hamiltonian cycles equals $\sum_{i=1}^{(n-1)!} \E[X_i] = p^n (n-1)!$.

Next, we show that $\sum_{i=1}^{(n-1)!} \E[X_i] \geq 1$ if $p = f(n)$ where $f(n) = \calO(\frac{1}{n})$. Starting from $\sum_{i=1}^{(n-1)!} \E[X_i] = p^n (n-1)!$, using the inequality $n! \geq (n/e)^n$ leads to
\begin{align*}
    p^n (n-1)! =\dfrac{p^n}{n}n! \geq \dfrac{p^n}{n} \left(\dfrac{n}{e}\right)^{n}
\end{align*}
Then, setting the RHS equal to 1 leads to
\begin{align*}
    \dfrac{p^n}{n} \left(\dfrac{n}{e}\right)^{n} = 1 
    \Leftrightarrow n \ln p + n \ln \dfrac{n}{e} = \ln n \Leftrightarrow \ln p = \ln \dfrac{e}{n} + \ln n^{\frac{1}{n}} \Leftrightarrow p = \dfrac{e}{n} n^{\frac{1}{n}}
\end{align*}
For large $n$, $\frac{1}{n}$ dominates $n^{\frac{1}{n}}$ and thus, the expected number of Hamiltonian cycle is larger than or equal to 1 with expected number of edges $n^2 p = \calO(n)$ 
\end{proof}

\begin{lemma}\label{lem:pattern}
There exists a parameterization of \modelname{} such that the sparsity patterns induced by the expected attention masks satisfy Assumption~\ref{ass:pattern}.
\end{lemma}
\begin{proof}
Here we show that a finite number of attention patterns each representable by the SBM given some number of clusters $k$ achieves the three conditions from Assumption~\ref{ass:pattern}. Here we use $p=3$ attention patterns together (shown in Figure~\ref{fig:approx_patterns}):
\begin{align*}
    \mathcal{A}^1_i &= \{i\} \cup \left\{j : \floor*{\frac{ik}{n}} = \floor*{\frac{jk}{n}} \;\forall j \in [n]\right\} \text{ for all } i\in [n] \\
    \mathcal{A}^2_i &= 
    \{i\}\cup\{n-k+1,\dots,n-1,n\} \text{ for all } i\in [n]\\
    \mathcal{A}^3_i &= 
    \begin{cases}
    \{i\} &\text{ if } i \leq n-k\\
    [n] &\text{ if } i > n-k\\
    \end{cases}
\end{align*}
\vspace{-3mm}

\begin{figure}[!t]
     \centering
     \begin{subfigure}{0.32\textwidth}
         \centering
         \includegraphics[trim={0mm 21mm 176mm 0mm},clip,width=.7\textwidth]{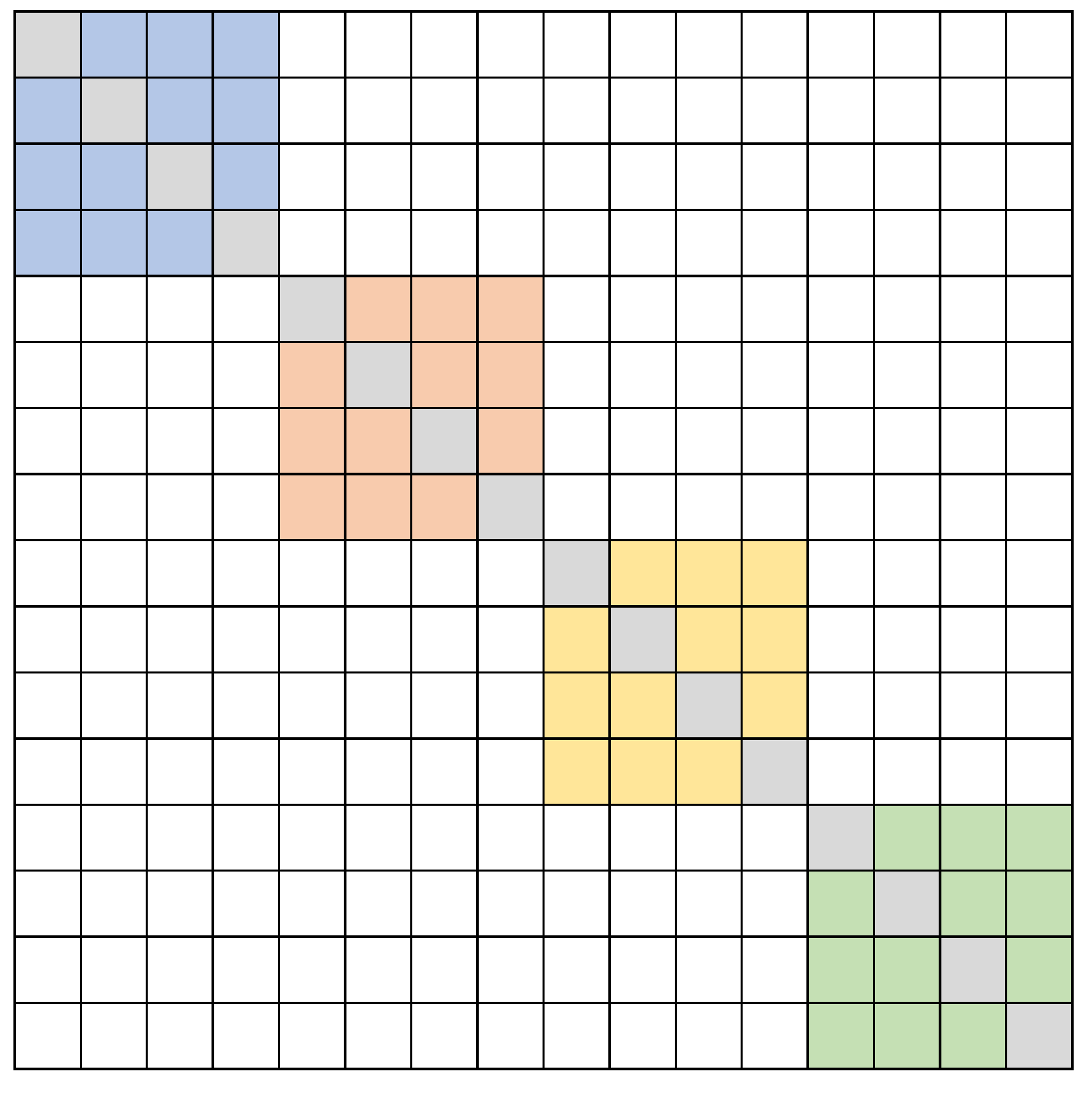}
         \caption{$\mathcal{A}^1$}
     \end{subfigure}
     \hfill
     \begin{subfigure}{0.32\textwidth}
         \centering
         \includegraphics[trim={0mm 21mm 176mm 0mm},clip,width=.7\textwidth]{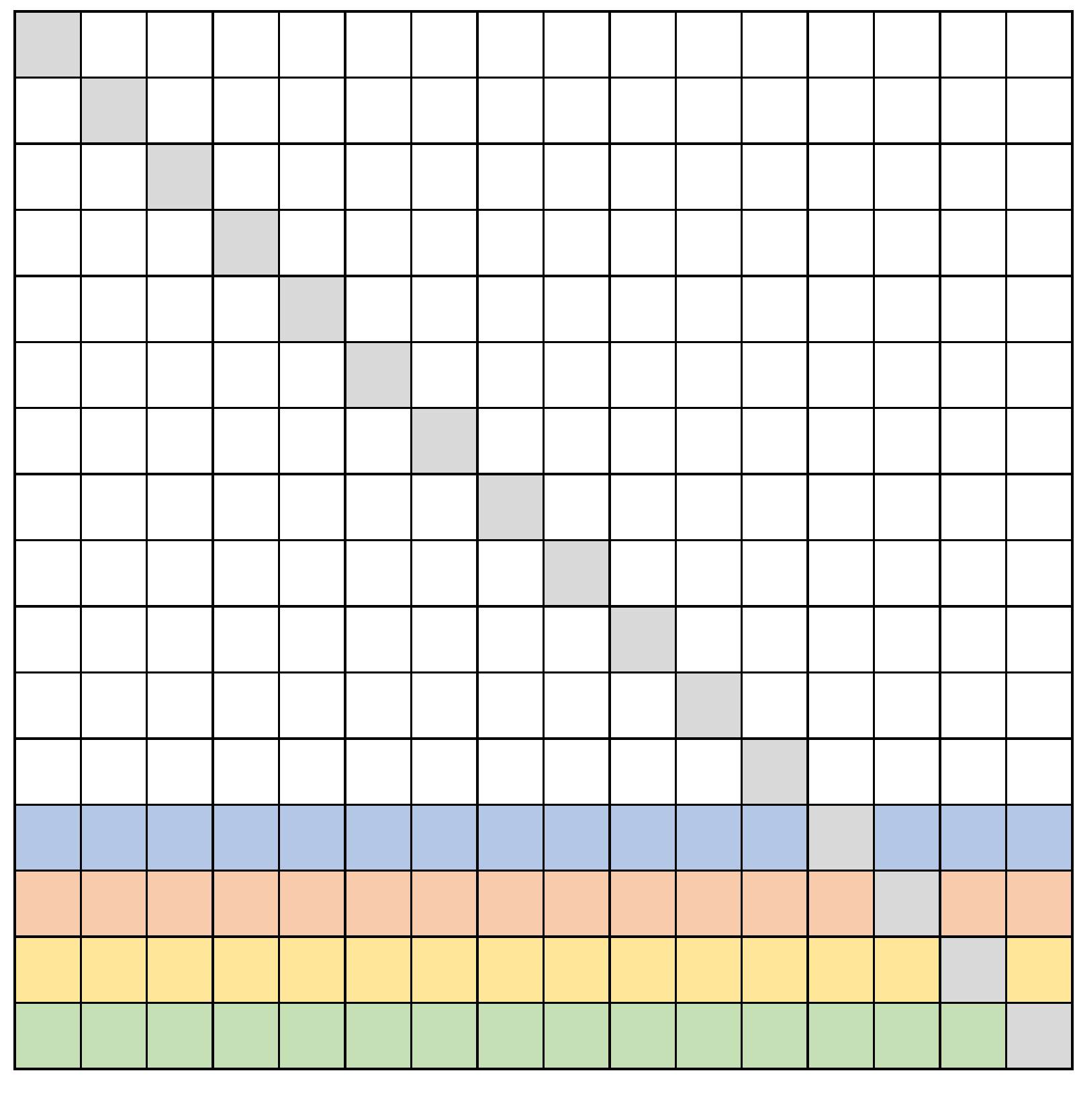}
         \caption{$\mathcal{A}^2$}
     \end{subfigure}
     \hfill
     \begin{subfigure}{0.32\textwidth}
         \centering
         \includegraphics[trim={0mm 21mm 176mm 0mm},clip,width=.7\textwidth]{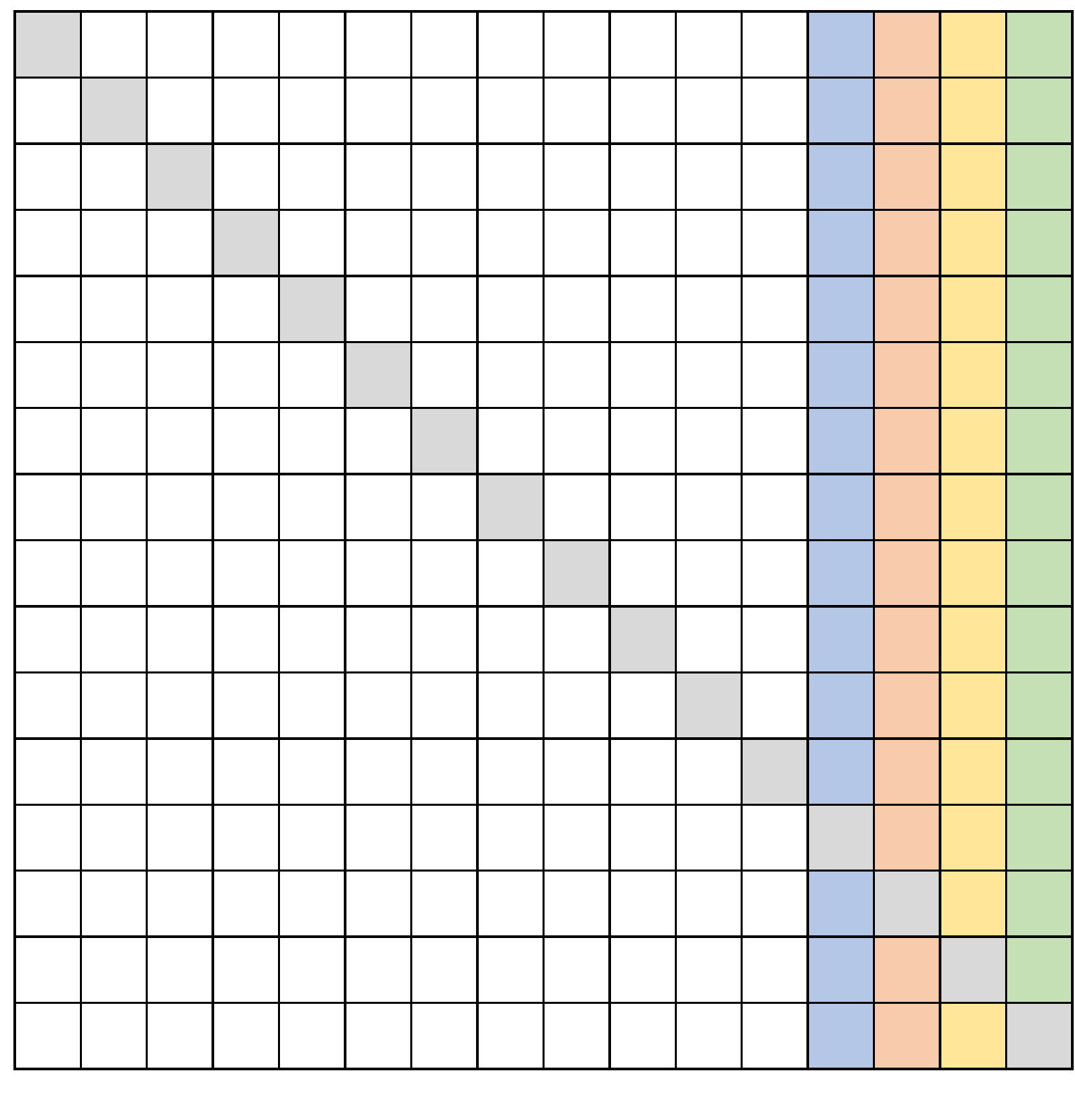}
         \caption{$\mathcal{A}^3$}
     \end{subfigure}
     \caption{Three sparsity patterns with $n=16$ and $k=4$. Grey-colored blocks on the diagonal indicate manually added self-loops. Any other color indicates a cluster.}
     \label{fig:approx_patterns}
     \vspace{-4mm}
\end{figure}

Intuitively speaking, $\mathcal{A}^1$ clusters all tokens into non-overlapping $k$ clusters, each with size $\frac{n}{k}$, and connects tokens together if they are within the same cluster. The other two patterns $\mathcal{A}^2$ and $\mathcal{A}^3$ adds $k$ global relay tokens for each cluster with edges going from and to all $n$ nodes, respectively. Note that all three patterns are easily representable from separate SBMs.

Then, we can show that these three patterns form directed graphs that together satisfy the three required conditions. Condition 1 is easily satisfied due to the manually added self-loops in all patterns. Condition 3 is also satisfied with $s=3$ as we have $k$ global relay tokens in both directions ($\mathcal{A}^2$ and $\mathcal{A}^3)$, connecting all pairs of tokens indirectly or directly. Lastly, Condition 2 can be satisfied by leveraging Lemma~\ref{lem:hamiltonian} and the global $k$ relay tokens: Lemma~\ref{lem:hamiltonian} states that each subgraph induced by each individual cluster in $\mathcal{A}^1$ has at least one Hamiltonian cycle with $\calO(n)$ number of edges in expectation. Then, a global Hamiltonian path can in $G^*$ can be constructed as follows:
\begin{itemize}
    \item Traverse through the first induced subgraph using its Hamiltonian cycle in $\mathcal{A}^1$, but without going back to the starting node.
    \item Move to the $n-k+1$ global relay token via the edge in $\mathcal{A}^2$, then move to any node in the second induced subgraph from node $n-k+1$ via an edge in $\mathcal{A}^3$.
    \item Traverse through the Hamiltonian cycle in the second induced subgraph, and repeat. 
\end{itemize}
This way, we can construct a global Hamiltonian path that visits all $n$ nodes, and all three conditions are met with $\calO(kn)$ number of edges in expectation.
\end{proof}

Combining Lemma~\ref{lem:pattern} together with Theorem 1 of Yun~et~al.~(2020)~\cite{yun2020} proves our main theorem below which states that \modelname{} is a universal approximator in expectation.

\begin{theorem}
Let $f \in \mathcal{F}$ be class of continuous sequence-to-sequence functions. Let $\mathcal{T}^{h,m,r}_{SBM}$ denote the class of SBM-Transformers with $h$ attention heads, $m$ head dimension, and $r$ dimensions in hidden layers. Then for any $\epsilon > 0$ and $1 \leq p < \infty$, there exists a function $g \in \mathcal{T}^{h,m,r}_{SBM}$ such that
\begin{align*}
    \int_\mathbb{D} \|f(\BX) - \mathbb{E}[g(\BX)]\|_p^p d\BX \leq \epsilon 
\end{align*}
\end{theorem}

\cutparagraphup
\section{Asymptotic Cost Analysis}
\addtocounter{footnote}{1}
Table~\ref{tab:asymptotics} shows the asymptotic computational cost and memory footprint of each step an attention head takes in \modelname{} given a single input. Assuming the number of clusters is significantly smaller than the sequence length, we find that both time and memory cost is mostly dominated by the computation of $\BQhat$ and $\BKhat$ when the sampled graph is sparse ({\it i.e.}, $m = \calO(n)$). \footnotetext{Walker's Alias Method~\cite{walker1977} used to sample nodes in \texttt{fastRG} requires $\calO(m+n\log n)$ operations, but the $\log n$ dependency is not visible in general. More information can be found in \cite{fastrg}}\addtocounter{footnote}{2}\footnotetext{We leverage highly optimized Generalized Sampled-Dense-Dense Matrix Multiplication (GSDDMM) operators provided by the Deep Graph Library~\cite{wang2019dgl} that avoids the $\calO(md)$ memory overhead.}
\begin{table}[!h]
    \centering
    \resizebox{.9\textwidth}{!}{\begin{tabular}{c|c c}
        \toprule
        Computation & Time & Memory \\
        \midrule
        Inputs $\BQ$, $\BK$, $\BV$, and $\BC$ & - & $\calO(nd+kd)$\\
        Node assignments $\BQhat$ and $\BKhat$ & $\calO(nd^2 + nkd)$ & $\calO(nd+kd+nk)$\\
        Inter-cluster probabilities $\BShat$  & $\calO(k^2 d)$ & $\calO(k^2)$\\
        Sampling from \texttt{fastRG}~\cite{fastrg} & $\calO(m+n)$\textsuperscript{1} & $\calO(m+nk+k^2)$\\
        Run GAT~\cite{gat} with edge-softmax & $\calO(md)$ & $\calO(m+nd)$\textsuperscript{2}\\
        \midrule
        Total & $\calO(md+nd^2+nkd+k^2d)$ & $\calO(m+nd+nk+kd+k^2)$ \\
        \bottomrule
    \end{tabular}
    }
    \vspace{2mm}
    \caption{Asymptotic costs of individual steps within the attention module of \modelname{}. The sequence length, number of edges, number of clusters, and head dimension are denoted as $n$, $m$, $k$, and $d$, respectively.}
    \label{tab:asymptotics}
    \vspace{-5mm}
\end{table}

A comparison of the overall cost of \modelname{} with those of other baselines is shown in Table~\ref{tab:asymptotics_baselines}. While its complexities most resemble those of Nystr\"omformer~\cite{nystromformer} when the sampled graphs are sparse, the cost of \modelname{} can exceed those of full-attention when the graph is dense, due to the additional computation in the $\text{MLP}_{d\to d}$ used to infer node-to-cluster memberships.

\begin{table}[h!]
    \centering
    \resizebox{.9\textwidth}{!}{\begin{tabular}{c|ccc}
        \toprule
        Model & Time & Memory\\
        \midrule
        Full-attention~\cite{transformer} & $\calO(n^2 d)$ & $\calO(n^2+nd)$\\ 
        \midrule
        Linearized~\cite{lineartransformer} & $\calO(nd^2)$ & $\calO(nd+d^2)$\\ 
        Reformer~\cite{reformer} & $\calO(nd+nk(4n/c)^2)$ & $\calO(nd+nk(4n/c)^2)$\\ 
        Performer~\cite{performer} & $\calO(nkd+kd^2)$ & $\calO(nk+nd)$\\ 
        Linformer~\cite{linformer} & $\calO(nkd+nk)$ & $\calO(nk+nd)$\\ 
        Nystr\"omformer~\cite{nystromformer} & $\calO(nkd+nk^2+k^3)$ & $\calO(nk+nd+kd+k^2)$\\ 
        \midrule
        \modelname{} (ours) & $\calO(md+nd^2+nkd+k^2d)$ & $\calO(m+nd+nk+kd+k^2)$\\ 
        \bottomrule
    \end{tabular}
    }
    \vspace{2mm}
    \caption{Asymptotic computational costs of different attention mechanisms. The $k$ term denotes different parameters for each model: number of clusters for \modelname{}, number of hashing rounds for Reformer~\cite{reformer}, number of random features for Performer~\cite{performer}, the projection rank for Linformer~\cite{linformer}, and the number of landmarks for Nystr\"omformer~\cite{nystromformer}. The additional $c$ term in Reformer~\cite{reformer} indicates the number of hashing chunks, set to $c = \calO(\frac{1}{n})$ as default.}
    \label{tab:asymptotics_baselines}
\end{table}

\section{Experiments}\label{app:results}

For reproducibility, we list the model and training hyperparameter settings used for each task in Table~\ref{tab:lra_hyperparam}. Note that for \modelname{}, we initialize the cluster-embeddings $\BC$ using the kaiming normal distribution~\cite{kaiming2015delving}, which results in an initial attention density of approximately 25\%. Tables~\ref{tab:lra_results} and \ref{tab:lra_lambda} provide the full LRA benchmark results with standard deviations in test-time accuracy and sparsity. As mentioned in the main paper, we find that manually fixing the self-loops in the sampled graphs slightly deteriorates performance, while it helps in proving theoretical expressibility.

\begin{table}[h!]
    \centering
    \resizebox{\textwidth}{!}{\begin{tabular}{l|r|rrrrr|rr}
        \toprule
        Parameter & \textsc{Synthetic} & \textsc{ListOps} & \textsc{Text} & \textsc{Retrieval} & \textsc{Image} & \textsc{Pathfinder} & \textsc{BERT} & \textsc{GLUE} \\
        \midrule
        \# of layers & 1 & 2 & 2 & 2 & 2 & 2 & 4 & 4\\ 
        \# of heads & 1 & 2 & 2 & 2 & 2 & 2 & 8 & 8\\ 
        Embedding dim. & 32 & 64 & 64 & 64 & 64 & 64 & 512 & 512\\ 
        Hidden dim. & 32 & 128 & 128 & 128 & 128 & 128 & 2048 & 2048\\ 
        Head dim. & 32 & 32 & 32 & 32 & 32 & 32 & 64 & 64\\ 
        Sequence len. & 256 & 2048 & 3072 & 4096 & 1024 & 1024 & 512 & 512\\ 
        Dropout & 0.0 & 0.1 & 0.1 & 0.1 & 0.1 & 0.1 & 0.1 & 0.1\\ 
        Attn. dropout & 0.0 & 0.1 & 0.1 & 0.1 & 0.1 & 0.1 & 0.1 & 0.1\\ 
        Pooling mode & N/A & MEAN & MEAN & MEAN & MEAN & MEAN & N/A & MEAN\\ 
        \# of classes & 2 & 10 & 2 & 2 & 10 & 2 & 50265 & 2 or 3\\ 
        Batch size & 256 & 128 & 128 & 32 & 1024 & 1024 & 256 & 32\\ 
        Learning rate & 1e-3 & 5e-4 & 5e-4 & 1e-4 & 5e-4 & 5e-4 & 1e-4 & 3e-5\\ 
        \# of training epochs & 2000 & 5000 & 20000 & 30000 & 35000 & 62400 & 50 & 5\\ 
        \bottomrule
    \end{tabular}
    }
    \vspace{2mm}
    \caption{Hyperparameter settings used synthetic, LRA, and GLUE experiments. For methods other than full attention~\cite{transformer}, we use 128 clusters for \modelname{}, 2 hashing rounds for Reformer~\cite{reformer}, 256 landmarks for Nystr\"omformer~\cite{nystromformer}, and 256 dimensions for Linformer~\cite{linformer} and Performer~\cite{performer}.}
    \label{tab:lra_hyperparam}
\end{table}

\begin{table}[h!]
    \centering
    \resizebox{\textwidth}{!}{\begin{tabular}{c|ccccc|c}
        \toprule
        Model & \textsc{ListOps}(2K) & \textsc{Text}(3K) & \textsc{Retrieval}(4K) & \textsc{Image}(1K) & \textsc{Pathfinder}(1K) & Avg. \\
        \midrule
        Full-attention~\cite{transformer} & 37.22$\pm$0.52 & 64.93$\pm$0.46 & 79.55$\pm$1.22 & 40.38$\pm$0.76 & 74.26$\pm$0.57 & 59.27$\pm$0.44\\ 
        \midrule
        Linearized~\cite{lineartransformer} & 37.46$\pm$0.57 & 64.90$\pm$0.49 & \bf 81.10$\pm$0.16 & 38.48$\pm$0.57 & 74.61$\pm$1.26 & 59.31$\pm$0.15\\ 
        Reformer~\cite{reformer} & 22.92$\pm$0.41 & 64.70$\pm$0.12 & 77.25$\pm$0.15 & \bf 43.65$\pm$0.16 & 70.28$\pm$1.45 & 55.76$\pm$0.29\\ 
        Performer~\cite{performer} & 18.25$\pm$0.12 & 65.00$\pm$0.50 & 79.01$\pm$1.66 & 39.80$\pm$0.46 & 70.79$\pm$1.26 & 54.57$\pm$0.55\\ 
        Linformer~\cite{linformer} & \bf 38.44$\pm$0.14 & 56.28$\pm$1.06 & 78.09$\pm$0.12 & 39.53$\pm$0.57 & 67.62$\pm$0.65 & 55.99$\pm$0.14\\ 
        Nystr\"omformer~\cite{nystromformer} & 37.22$\pm$0.51 & \underline{65.46$\pm$0.40} & 79.35$\pm$0.40 & \underline{43.07$\pm$0.42} & 71.97$\pm$1.30 & \underline{59.41$\pm$0.12}\\ 
        \midrule
        \multirow{2}{*}{\modelname{} ($+\BI$)} & \underline{37.60$\pm$0.38} & 64.09$\pm$1.39 & 79.74$\pm$0.27 & 40.64$\pm$0.72 & \underline{74.93$\pm$0.32} & \multirow{2}{*}{59.40$\pm$0.20}\\
        & \underline{(24.64$\pm$2.49\%)} & (25.64$\pm$0.64\%) & (24.26$\pm$5.21\%) & (24.54$\pm$3.98\%) & \underline{(23.84$\pm$3.59\%)} & \\
        \multirow{2}{*}{\modelname{} ($+\Bzero$)} & 37.45$\pm$0.44 & \bf 65.79$\pm$0.27 & \underline{80.00$\pm$0.21} & 41.31$\pm$0.35 & \bf 75.12$\pm$0.49 & \multirow{2}{*}{\bf 59.93$\pm$0.35}\\ 
        & (20.09$\pm$15.71\%) & \bf (26.10$\pm$0.01\%) & \underline{(29.46$\pm$3.84\%)} & (20.49$\pm$11.43\%) & \bf (18.56$\pm$0.52\%) & \\
        \bottomrule
    \end{tabular}
    }
    \vspace{2mm}
    \caption{LRA benchmark results. Bold and underlined results indicate best and 2nd best test accuracy for each task, respectively. Numbers enclosed in parentheses for \modelname{} indicate the density of graphs sampled during test time averaged across all attention heads. For the \modelname{} models, $(+\BI)$ indicates that self-loops are manually fixed while $(+\Bzero)$ indicates model without the modification.}
    \label{tab:lra_results}
\end{table}

\begin{table}[h!]
    \centering
    \resizebox{\textwidth}{!}{\begin{tabular}{c|ccccc|c}
        \toprule
        $\lambda$ & \textsc{ListOps}(2K) & \textsc{Text}(3K) & \textsc{Retrieval}(4K) & \textsc{Image}(1K) & \textsc{Pathfinder}(1K) & Avg. \\
        \midrule
        \multirow{2}{*}{0} & 37.45$\pm$0.44 & \bf 65.79$\pm$0.27 & \underline{80.00$\pm$0.21} & 41.31$\pm$0.35 & \underline{75.12$\pm$0.49} & \multirow{2}{*}{\underline{59.93$\pm$0.35}}\\
        & (20.09$\pm$15.71\%) & \bf (26.10$\pm$0.01\%) & \underline{(29.46$\pm$3.84\%)} & (20.49$\pm$11.43\%) & \underline{(18.56$\pm$0.52\%)} & \\
        \multirow{2}{*}{$10^{-4}$} & 37.76$\pm$0.60 & 65.48$\pm$0.86 & 79.93$\pm$0.16 & \underline{41.35$\pm$0.35} & \bf 75.46$\pm$0.46 & \multirow{2}{*}{\bf 60.00$\pm$0.36}\\ 
        & (10.48$\pm$7.58\%) & (26.26$\pm$0.53\%) & (24.62$\pm$3.19\%) & \underline{(10.70$\pm$8.49\%)} & \bf (5.16$\pm$1.17\%) & \\
        \multirow{2}{*}{$10^{-3}$} & \bf 38.23$\pm$0.63 & 65.18$\pm$0.46 & 80.00$\pm$0.99 & 41.17$\pm$0.53 & 74.49$\pm$0.74 & \multirow{2}{*}{59.81$\pm$0.48}\\ 
        & \bf (10.46$\pm$7.26\%) & (26.03$\pm$0.06\%) & (21.70$\pm$2.68\%) & (24.60$\pm$8.61\%) & (3.82$\pm$0.52\%) & \\
        \multirow{2}{*}{$10^{-2}$} & \underline{38.20$\pm$0.29} & \underline{65.59$\pm$0.24} & \bf 80.44$\pm$1.24 & \bf 42.20$\pm$0.64 & 72.79$\pm$0.80 & \multirow{2}{*}{59.84$\pm$0.42}\\ 
        & \underline{(2.95$\pm$0.88\%)} & \underline{(22.43$\pm$1.73\%)} & \bf (6.99$\pm$2.28\%) & \bf (3.95$\pm$0.68\%) & (3.76$\pm$0.27\%) & \\
        \multirow{2}{*}{$10^{-1}$} & 37.76$\pm$0.83 & 64.48$\pm$0.58 & 79.46$\pm$0.47 & 41.35$\pm$0.40 & 73.79$\pm$0.07 & \multirow{2}{*}{59.37$\pm$0.37}\\ 
        & (1.15$\pm$0.15\%) & (10.62$\pm$2.74\%) & (2.49$\pm$0.58\%) & (1.33$\pm$0.37\%) & (2.61$\pm$0.22\%) & \\
        \bottomrule
    \end{tabular}
    }
    \vspace{2mm}
    \caption{LRA benchmark results of \modelname{} with increasing density regularization weight $\lambda$. Applying a density regularizer helps in encouraging sparser attention patterns which induce less computational cost, while retaining competitive performance.}
    \label{tab:lra_lambda}
\end{table}

\vfill\pagebreak
Lastly, we qualitatively analyze which inputs lead to sparse or dense attention in \modelname{}. For easy visualization of attention densities, we choose two image-based tasks in LRA, \textsc{Pathfinder} and \textsc{Image}. We pick two model checkpoints that performed best on each of the two tasks under graph density regularization, one trained with $\lambda = 10^{-4}$ for \textsc{Pathfinder} and another trained with $\lambda = 10^{-2}$ for \textsc{Image}, and run predictions on the respective test sets. Figures~\ref{fig:pathfinder_patterns} and \ref{fig:image_patterns} show the head-wise attention densities per input at different levels. 

\begin{figure}[t!]
     \centering
     \begin{subfigure}{\textwidth}
         \centering
         \includegraphics[width=\textwidth]{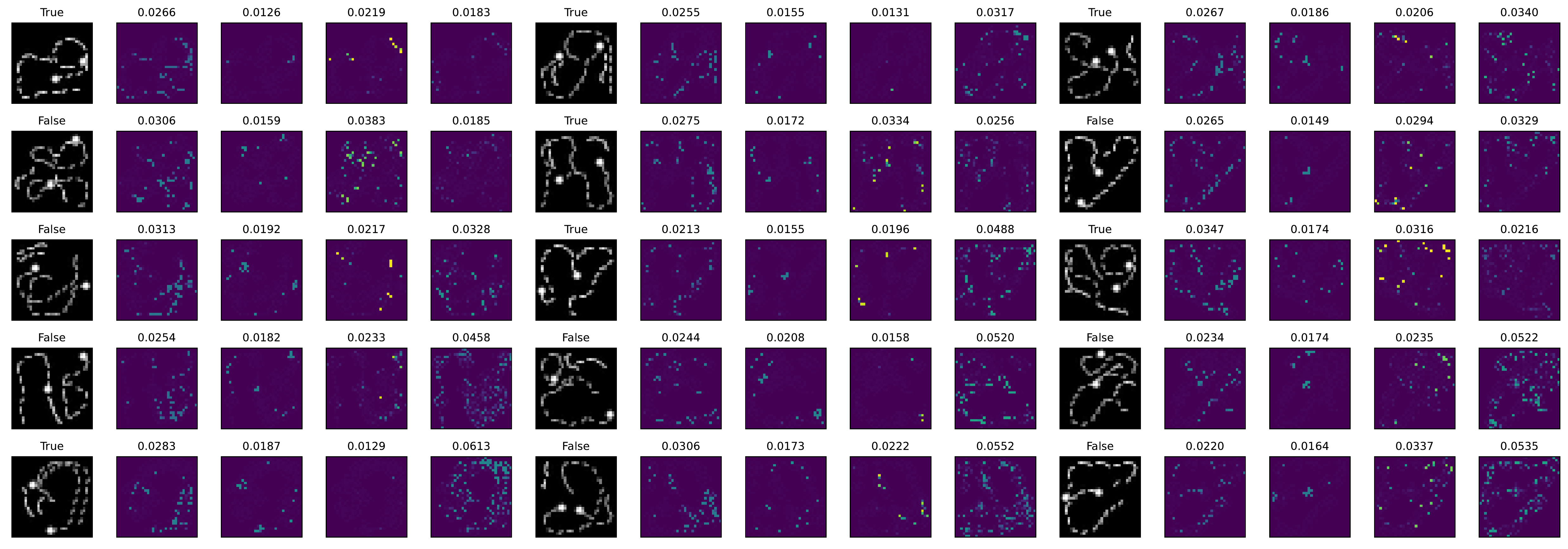}
         \caption{Examples with low attention density}
         \label{fig:pathfinder_patterns_sparse}
     \end{subfigure}
     \begin{subfigure}{\textwidth}
         \centering
         \includegraphics[width=\textwidth]{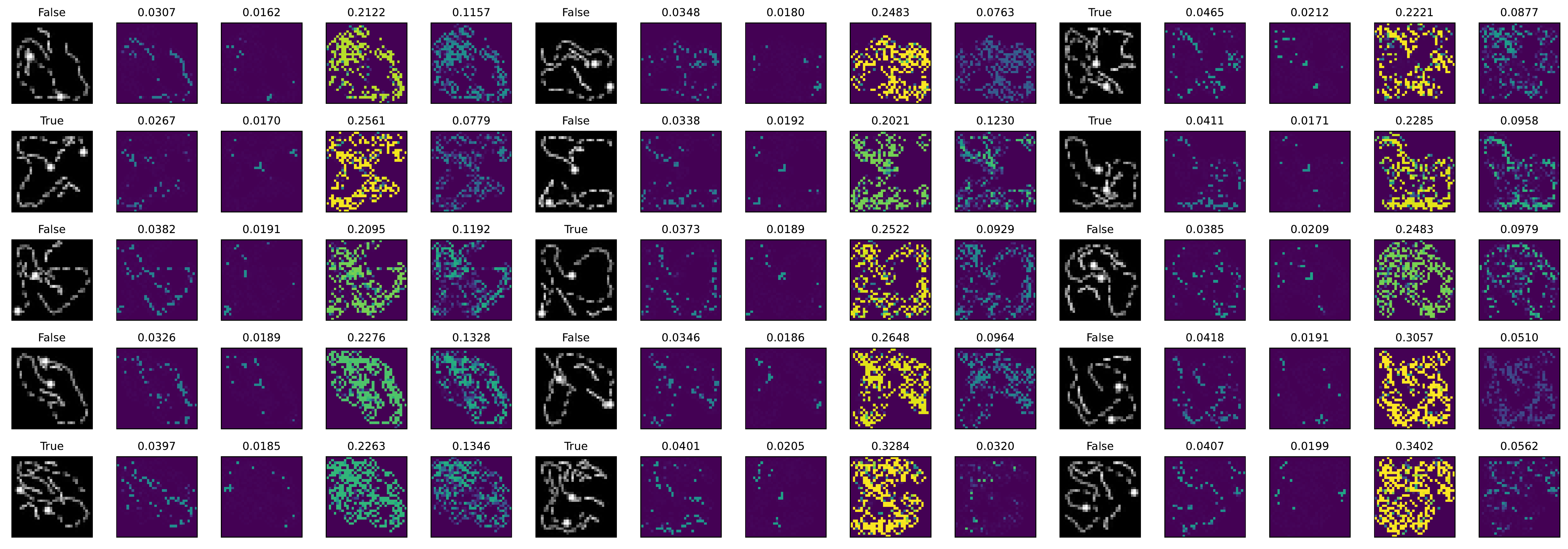}
         \caption{Examples with high attention density}
         \label{fig:pathfinder_patterns_dense}
     \end{subfigure}
     \caption{Attention density plots within individual attention heads given inputs from the LRA \textsc{Pathfinder} test set. All examples shown are from a subset of the test set that the model has predicted correctly. For each set of 5 images, the leftmost image shows the original input image of which the title shows the ground-truth label. To its right are attention density plots from two heads of the first layer followed by those from two heads of the second layer. Above each plot is the actual numeric attention density between 0 and 1. The color in each pixel indicates how many other pixels attend to that particular pixel (a color closer to bright yellow indicates more attention).}
     \label{fig:pathfinder_patterns}
     \vspace{-15mm}
\end{figure}

In Figure~\ref{fig:pathfinder_patterns}, the second layer shows large variance in attention density across different \textsc{Pathfinder} inputs, while the first layer remains sparse overall. With some exceptions, we find that the attention density of this layer is somewhat correlated with the difficulty of each input. Figure~\ref{fig:pathfinder_patterns_sparse} shows visually easy inputs with near-perpendicular intersections or no intersection at all, allowing correct predictions with less than 5\% average attention density. On the other hand, Figure~\ref{fig:pathfinder_patterns_dense} shows examples with harder difficulty, due to having more lines and convoluted intersections. We can see that the model uses much denser attention in such cases, and thus conjecture that the model is adaptively choosing to look at more pixel-to-pixel interactions in response to the complexity of the input.

Figure~\ref{fig:image_patterns} also shows a clear distinction between images that induce different levels of attention density. Under regularization, the first layer of \modelname{} focuses attention onto dark areas in the image as shown in Figure~\ref{fig:image_patterns_dense}, using the contrast in the image for better prediction. When the image has high overall intensity as in Figure~\ref{fig:image_patterns_sparse}, however, the model uses less than 3\% attention on average, focusing most of the prediction onto the skip-connections, FFNs, and a small number of pixel-to-pixel interactions. Considering that this model achieves a competitive 42.20\% accuracy, this shows that \modelname{} can well balance the tradeoff between computational cost vs. performance, further supporting the power of our adaptively sparse attention module.

\pagebreak
\begin{figure}[t!]
     \centering
     \begin{subfigure}{\textwidth}
         \centering
         \includegraphics[width=\textwidth]{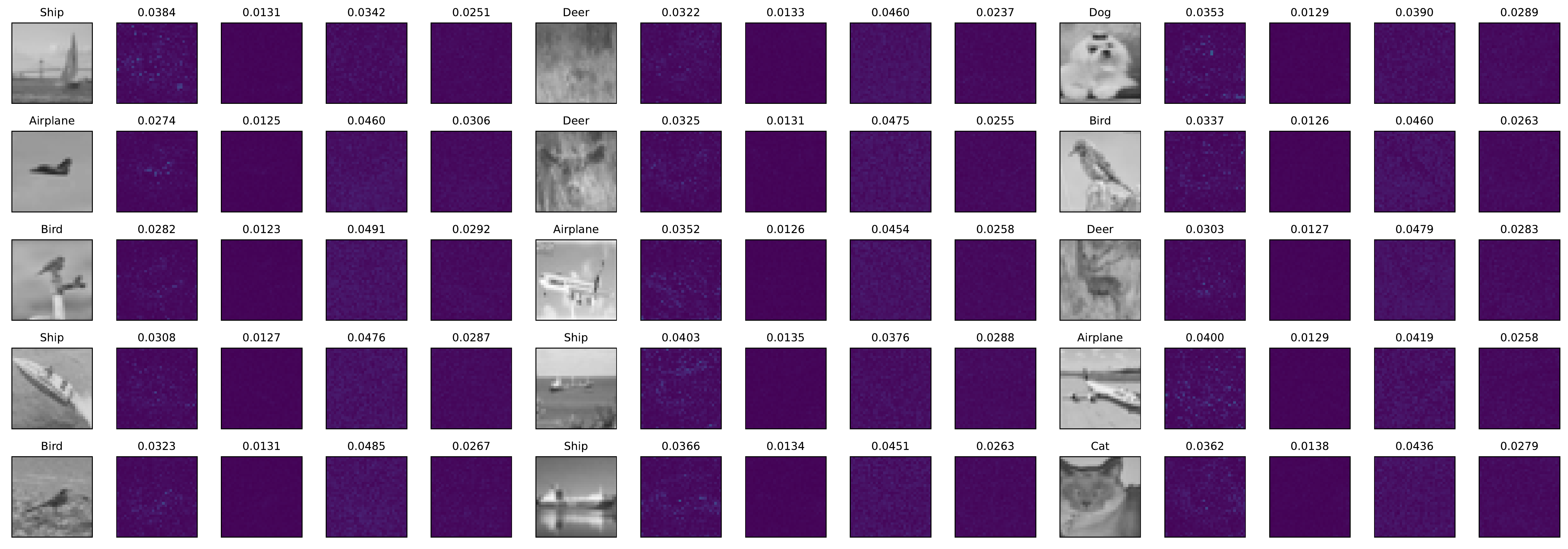}
         \caption{Examples with low attention density}
         \label{fig:image_patterns_sparse}
     \end{subfigure}
     \begin{subfigure}{\textwidth}
         \centering
         \includegraphics[width=\textwidth]{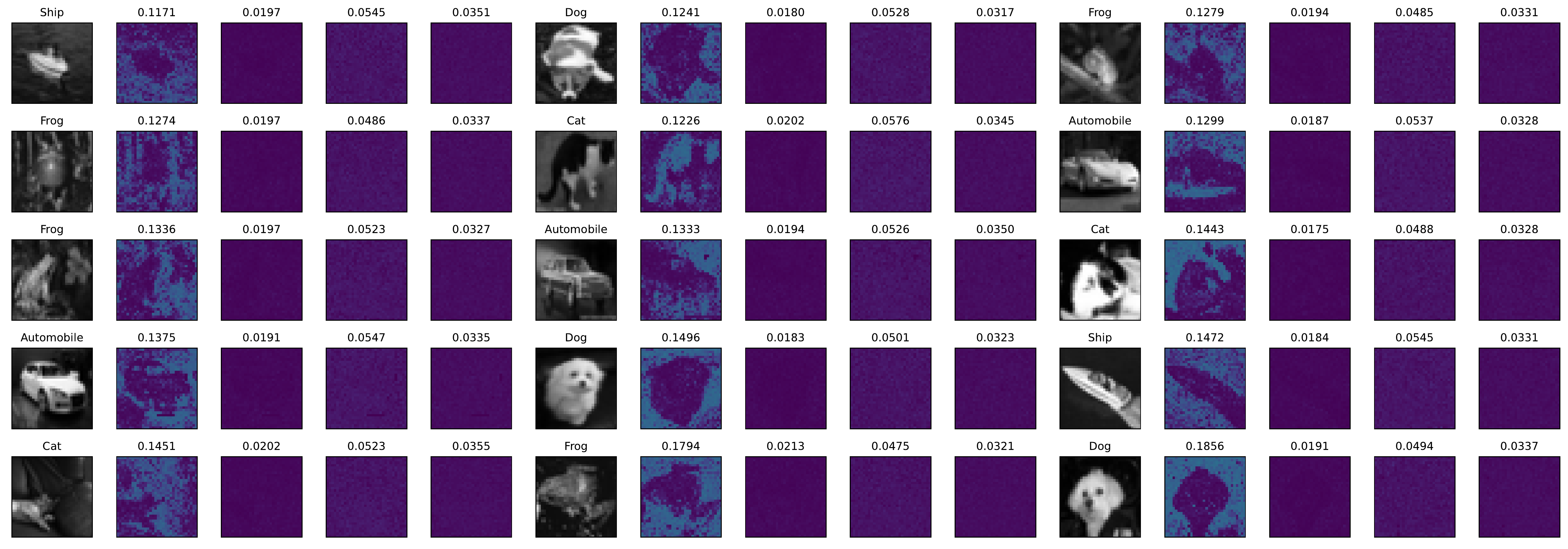}
         \caption{Examples with high attention density}
         \label{fig:image_patterns_dense}
     \end{subfigure}
     \caption{Similar visualization as Figure~\ref{fig:pathfinder_patterns} for the LRA \textsc{Image} test set.}
     \label{fig:image_patterns}
\end{figure}

\end{document}